\documentclass{elsarticle}
\usepackage{hyperref}

\usepackage[table]{xcolor}
\usepackage{graphicx}
\usepackage{times}
\usepackage{amsmath,amsfonts}
\usepackage{algorithm}
\usepackage{tikz}
\usetikzlibrary{bayesnet}

\newcommand{\dd}{\,d}
\newcommand{\Prob}{\operatorname{P}}
\newcommand{\given}{\operatorname{\mid}}
\newcommand{\X}{\mathbf{X}}
\newcommand{\G}{\mathcal{G}}
\newcommand{\D}{\mathcal{D}}

\newcommand{\PXi}{\Pi_{X_i}}
\newcommand{\XPi}{X_i \given \PXi}
\newcommand{\Dirichlet}{\operatorname{Dirichlet}}
\newcommand{\Categorical}{\operatorname{Categorical}}
\newcommand{\btheta}{\boldsymbol{\theta}}
\newcommand{\tXi}{\btheta_{\XPi}}

\newcommand{\tXPiF}{\btheta_{X_i, \PXi\given F}}
\newcommand{\tXPif}{\btheta_{X_i, \PXi}^f}

\newcommand{\htXPif}{\boldsymbol{\widehat{\theta}}_{X_i, \PXi}^f}

\newcommand{\tXijf}{\btheta_{X_i \given j}^f}
\newcommand{\tXij}{\btheta_{X_i \given j}}

\newcommand{\htXij}{\boldsymbol{\widehat{\theta}}_{X_i \given j}}

\newcommand{\balpha}{\boldsymbol{\alpha}}
\newcommand{\bai}{\boldsymbol{\alpha}_i}
\newcommand{\aijk}{\alpha_{ijk}}

\newcommand{\BD}{\operatorname{BD}}
\newcommand{\BHD}{\operatorname{BHD}}
\newcommand{\E}{\operatorname{E}}
\newcommand{\bkappa}{\boldsymbol{\kappa}}
\newcommand{\kijk}{\kappa_{ijk}}
\newcommand{\hkijk}{\widehat{\kappa}_{ijk}}
\newcommand{\hkij}{\widehat{\kappa}_{ij}}
\newcommand{\bnu}{\boldsymbol{\nu}}
\newcommand{\forallF}{f = 1, \ldots, |F|}

\newtheorem{thm}{Theorem}
\newtheorem{lemma}[thm]{Lemma}
\newproof{proof}{Proof}

\begin{document}

\begin{frontmatter}

\title{A Bayesian Hierarchical Score for Structure Learning from Related Data Sets}

\author{Laura Azzimonti\corref{mycorrespondingauthor}}
\cortext[mycorrespondingauthor]{Corresponding author}
\ead{laura.azzimonti@idsia.ch}
\author{Giorgio Corani}
\author{Marco Scutari}
\address{Istituto Dalle Molle di Studi sull'Intelligenza Artificiale (IDSIA),
  USI/SUPSI, Lugano, Switzerland}

\begin{abstract}
  Score functions for learning the structure of Bayesian networks in the
  literature assume that data are a homogeneous set of observations; whereas it
  is often the case that they comprise different related, but not homogeneous,
  data sets collected in different ways.
  In this paper we propose a new Bayesian Dirichlet score, which we call
  Bayesian Hierarchical Dirichlet (BHD). The proposed score is based on a
  hierarchical model that pools information across data sets to learn a single
  encompassing network structure, while taking into account the differences in
  their probabilistic structures. We derive a closed-form expression for BHD
  using a variational approximation of the marginal likelihood, we study the
  associated computational cost and we evaluate its performance using simulated
  data. We find that, when data comprise multiple related data sets, BHD
  outperforms the Bayesian Dirichlet equivalent uniform (BDeu) score in terms of
  reconstruction accuracy as measured by the Structural Hamming distance, and
  that it is as accurate as BDeu when data are homogeneous. This improvement is
  particularly clear when either the number of variables in the network or the
  number of observations is large. Moreover, the estimated networks are sparser
  and therefore more interpretable than those obtained with BDeu thanks to a
  lower number of false positive arcs.
\end{abstract}


\end{frontmatter}


\section{Introduction}
\label{sec:intro}

Investigating challenging problems at the forefront of science increasingly
requires large amounts of data that can only be gathered through collaborations
between several institutions. This naturally leads to heterogeneous data sets
that are in fact the collation of related, but not identical, subsets of data
that will necessarily differ in the details of how they are collected. Examples
can be found in multi-centre clinical trials, in which protocols are applied in
slightly different ways to different patient populations 
\cite{multicenter,spiegelhalter}; population genetics, which studies the 
architecture of phenotypic traits across populations and their evolution \cite{goddard2,reviewer,makowsky,campos}; ecology and environmental sciences, 
which produce different patterns of measurement errors and limitations in 
different environments \cite{kazakhstan,vitolo17,ecology}. A common goal in 
analysing  these complex data is to construct a mechanistic model that elucidates
the interplay between different elements under investigation, either as a step 
towards building a causal model or to perform accurate prediction from a 
purely probabilistic perspective.

The task of efficiently modelling such related data sets is usually 
tackled by hierarchical models \cite{gelman}, which pool
the information common to the different subsets of the data while 
encoding the information that is specific to each subset. For instance multilevel regression models  estimate  the conditional distribution of the  response variables in these cases.

Bayesian networks (BNs) \cite{koller} provide a rigorous approach for 
modelling joint distributions, by representing variables as nodes and
probabilistic dependencies as arcs in a graph. 
They can be used for both
causal and predictive modelling.
To the best of our knowledge,
however, no method has been proposed in the literature to combine these two
approaches to learn a single BN structure from a set of related data sets and
get the best of both worlds.

Available methods focus on learning an ensemble of BNs that
have similar structures by penalising differences in their arc sets
\cite{niculescu2007inductive, Oates2016}. Parameter learning from related data
sets has been investigated in \cite{demichelis2006hierarchical} for Gaussian BNs
and in \cite{malovini2012hierarchical} for discrete BNs. However, they only
consider a naive Bayes structure and they initialise their hyperprior with
maximum likelihood point estimates.

In this paper, we show how to learn the structure of a BN from related data
sets, containing the same variables, by building on our previous work on
parameter learning in \cite{azzimonti19}. 
The proposed approach is particularly suited to deal with multiple related data sets characterised by few observations per data set or by an unbalanced number of observations across data sets. In these settings, it is important to share information across data sets to obtain robust estimates of both the parameters and the structure of the BN.
First, we
briefly introduce BNs and hierarchical models in the context of discrete data as
well as prior work on parameter learning from related data sets in Section \ref{sec:background}. We  propose
a score function for related data sets in Section \ref{sec:hBN} and we study the
associated computational complexity in both a theoretical and empirical way in
Section \ref{sec:complexity}. Then, we show an example of structure learning by means of BHD in Section \ref{sec:exe} and 
we  study its performance on different
simulation studies in Section \ref{sec:simulations}. Finally, we discuss our
results and possible future research directions in Section \ref{sec:concl}.

\section{Background and Notation}
\label{sec:background}

Bayesian networks (BNs) are a class of graphical models that use a directed
acyclic graph (DAG) $\G$ to model a set of random variables $\X = \{X_1, \ldots,
X_N\}$: each node is associated with one $X_i \in \X$ and arcs represent direct
dependence relationships. Graphical separation of two nodes implies the
conditional independence of the corresponding random variables. In principle,
there are many possible choices for the joint distribution of $\X$; literature
has focused mostly on discrete BNs \cite{heckerman}, in which both $\X$ and the
$X_i$ are categorical (multinomial) random variables. Other possibilities
include Gaussian BNs and conditional linear Gaussian BNs \cite{lauritzen},
which include both discrete and Gaussian BNs as particular cases.

The task of learning a BN from a data set $\D$ of $n$ observations is performed
in two steps in an inherently Bayesian fashion:
\begin{align}
\label{eq:lproc}
  \underbrace{\Prob(\G, \Theta \given \D)}_{\text{learning}} =
    \underbrace{\Prob(\G \given \D)}_{\text{structure learning}} \cdot
    \underbrace{\Prob(\Theta \given \G, \D)}_{\text{parameter learning}},
\end{align}
where $\Theta$ are the parameters of $\X$.
\emph{Structure learning} consists in finding the DAG $\G$ that encodes the
dependence structure of the data. In this paper we will focus on score-based
algorithms, which are typically heuristic search algorithms that use a
goodness-of-fit score such as BIC \cite{schwarz} or the Bayesian Dirichlet
equivalent uniform (BDeu) marginal likelihood \cite{heckerman} to find an
optimal $\G$.
\emph{Parameter learning} involves the estimation of the parameters $\Theta$
given the DAG $\G$ learned in the first step. Thanks to the Markov property,
this step is computationally efficient because if the data are complete the
\emph{global distribution} of $\X$ decomposes into
\begin{equation}
\label{eq:parents}
  \Prob(\X \given \G) = \prod_{i=1}^N \Prob(\XPi)
\end{equation}
and the \emph{local distribution} associated with each node $X_i$ depends only
on the configurations of its parents $\PXi$. If we estimate the parameters 
$\Theta = \{ \Theta_{X_1}, \ldots, \Theta_{X_N}\}$ in such a way that they 
are independent across local distributions, parameter learning simplifies 
into a collection of low-dimensional estimation problems for the $\Theta_{X_i}$ 
associated with each $\XPi$ given the data available for those variables. 

\subsection{Classic Multinomial-Dirichlet Parameterisation}

In the case of discrete BNs, we assume that each $\XPi$ follows a categorical
distribution for each configuration of $\PXi$. Hence the parameters of $\XPi$
are the conditional probabilities $\tXi =\{\tXij, j=1,\ldots,|\PXi|\}$, whose
$k$th element corresponds to $\Prob(X_i = k \given \PXi = j)$,
for which we assume a conjugate Dirichlet prior:
\begin{align}
  \left. \tXi \,\right| \bai &\sim \Dirichlet(\bai)  \nonumber \\
  X_i \left|\, \PXi ,\tXi \right. &\sim \Categorical\left(\tXi\right),
\label{eq:BDeu}
\end{align}
where $\bai = \{\aijk,  j = 1, \ldots, |\PXi|; k = 1,
\ldots, |X_i|\}$, with $i = 1, \ldots, N$, is a hyperparameter vector defined over a simplex with sum
$\sum_{jk} \aijk = s_i > 0$. The posterior estimator of $\tXi$ is:
\begin{align}
  &\left[ \htXij\right]_{k} = \frac{\aijk + n_{ijk}}{\alpha_{ij} + n_{ij}},& \text{where }
  &n_{ij} = \sum\nolimits_k n_{ijk},\ \  \alpha_{ij} = \sum\nolimits_k \aijk,
\label{eq:postprob}
\end{align}
and $n_{ijk}$ represents the number of observations for which $X_i=k$ and
$\PXi=j$. It is common to set $\aijk = s_i / (|X_i||\PXi|)$ with the same
\emph{imaginary sample size} $s_i = s$ for all $X_i$.

In the context of structure learning, we have $\Prob(\G \given \D) \propto
\Prob(\D \given \G) \Prob(\G)$ and we can use $\Prob(\D \given \G)$ as a score
function. (Implicitly, we are saying that $\Prob(\G) \propto 1$ by disregarding
it while still searching for the \emph{maximum a posteriori} DAG.) Assuming 
\emph{positivity} ($\tXi > 0$), \emph{parameter independence}
(columns of $\tXi$ associated with different parent configurations are
independent), \emph{parameter modularity} ($\tXi$ associated with different
nodes are independent) and \emph{complete data}, \cite{heckerman} derived a
closed form expression for $\Prob(\D \given \G)$ known as the \emph{Bayesian
Dirichlet} (BD) family of scores:
\begin{equation}
  \BD(\G, \D; \balpha) =
  \prod_{i=1}^N \BD\left(\XPi; \bai\right) =
  \prod_{i=1}^N \prod_{j = 1}^{|\PXi|}
    \left[
      \frac{\Gamma(\alpha_{ij})}{\Gamma(\alpha_{ij} + n_{ij})}
      \prod_{k=1}^{|X_i|} \frac{\Gamma(\aijk + n_{ijk})}{\Gamma(\aijk)}
    \right].
\label{eq:bd}
\end{equation}
Choosing again $\aijk = s / (|X_i||\PXi|)$ gives the \emph{Bayesian Dirichlet
equivalent uniform} (BDeu) score. A default value of $s = 1$ has been
recommended by \cite{ueno}. Assuming a uniform prior for both $\G$ and $\tXi$
is common in the literature, even if they  can have serious impact on the
accuracy of the learned structures \cite{pgm16}, especially for sparse data
that are likely to lead to violations of the positivity assumption
\cite{behaviormetrika18}. These assumptions are taken to represent lack of
prior knowledge, and they make BDeu the only BD score giving the 
same score value to BNs in the same equivalence class (\emph{score-equivalence} 
\cite{chickering}). Equivalence classes are characterised by the skeleton of 
$\G$ (its underlying undirected graph) and its v-structures (patterns of arcs 
of the type $X_j \rightarrow X_i \leftarrow X_k$, with no arc between $X_j$ 
and $X_k$), and group DAGs that encode the same global distribution.

\subsection{Hierarchical Multinomial-Dirichlet Parameterisation for Related
Data Sets}

The classic Multinomial-Dirichlet model in \eqref{eq:BDeu} can be extended to
handle related data sets by treating it as a particular case of the hierarchical
Multinomial-Dirichlet (hierarchical MD) model presented in \cite{azzimonti19}.
For this purpose, we introduce an auxiliary variable $F$ which identifies the
$|F|$ related data sets. Assuming that the data sets contain the same variables
and that $F$ is always observed, we can learn a BN with a common structure $\G$
but with different parameter estimates for each related data set.

For simplicity, we apply the hierarchical model independently to each local
distribution to estimate the joint distribution of $(X_i, \PXi)$ conditional on
$F$, $\tXPiF=\{\tXPif, f=1, \ldots, |F|\}$, by  pooling information between
different data sets. The resulting hierarchical model is shown in the top panel
of Figure \ref{fig:BN_hier_multi}. Specifically, for each node $X_i$ we assume
$\bai$ to be a latent random vector and we add a Dirichlet hyperprior to make
$\tXPif$ a mixture of Dirichlet distributions:

\begin{align}
  \bai \given s_i, \balpha_{0,i}
    &\sim s_i\cdot \Dirichlet(\balpha_{0,i}),  & \nonumber \\
  \left. \tXPif \right| \bai
    &\sim \Dirichlet(\bai) & \forallF, \label{eq:hier} \\
  X_i,\PXi  \left|\, F=f, \tXPif \right.
    &\sim \Categorical\left(\tXPif\right) & \forallF, \nonumber
\end{align}
where this time $\bai = \{ \aijk \}$ is a latent random vector defined over a
simplex of dimension $|X_i||\PXi|-1$ with sum $s_i$. The new hyperparameters of this model are the imaginary
sample size $s_i$ and the parameter vector $\balpha_{0,i}$, which in turn is
defined over a simplex with sum $s_{0,i}$. 
The two parameters $s_0$ and $s_i$ control respectively 
the concentration of the 
$\bai$ random vectors around the discrete distribution $\balpha_{0,i}/s_0$ and the variance of 
$\tXPif$, with
$f=1, \ldots, |F|$,
around the normalised random vector $\bai/s_i$.
Larger values of $s_i$ yield  $\tXPif$ 
that are more similar to each other, while
larger values $s_0$ provide $\tXPif$ closer
to the uniform distribution.
In the following we will drop both
hyper-parameters from the notation for brevity.

\begin{figure}[t]
	\begin{center}
	\begin{tikzpicture}[x=1.7cm,y=1.8cm]


  \node[obs]                   (X)      {$ X_i, \Pi_{X_i} | f$} ;
  \node[latent, left=of X]    (theta)  {$\boldsymbol{\theta}_{ X_i, \Pi_{X_i}}^f$};
  \node[latent, left=of theta] (alpha) {$\boldsymbol{\alpha}_i$};
  \node[const, left=of alpha] (alpha0) {$\boldsymbol{\alpha}_{0,i}$};
  \node[latent, below=of X]    (theta_vi)  {$\boldsymbol{\theta}_{ X_i, \Pi_{X_i}}^f$}; %
   \node[latent, left=of theta_vi]    (nu_vi)  {$\boldsymbol{\nu}^f_{i}$}; %
 \node[latent, left=of nu_vi] (alpha_vi) {$\boldsymbol{\alpha}_i$};
  \node[latent, left=of alpha_vi] (alpha0_vi) {$\tau_i \boldsymbol{\kappa}_i$};

  \factor[left=of X]     {X-f}     {Cat.} {} {} ;
   \factor[left=of theta] {theta-f} {Dir.} {} {} ; %
\factor[left=of alpha] {alpha-f} {Dir.} {} {} ; %
 \node[const, below=of alpha-f, yshift=+1cm] (s) {$s_i$};
 \factor[left=of theta_vi] {theta-f_vi} {Dir.} {} {} ; %
 \factor[left=of alpha_vi] {alpha-f_vi} {Dir.} {} {} ; %
 \node[const, below=of alpha-f_vi,yshift=+1cm] (s_vi) {$s_i$};

    \factoredge {theta}    {X-f}     {X} ; %
  \factoredge {alpha}    {theta-f}     {theta} ; %
  \factoredge {alpha0}    {alpha-f}     {alpha} ; %
  \factoredge    {alpha-f} {s}  {alpha-f}; %
  \factoredge {nu_vi}    {theta-f_vi}     {theta_vi} ; %
  \factoredge {alpha0_vi}    {alpha-f_vi}     {alpha_vi} ; %
  \factoredge    {alpha-f_vi} {s_vi}  {alpha-f_vi}; %

  \plate {plate1} { %
    (X)
    (theta)
  } {$f = 1, \ldots, |F|$}; %
  \plate {plate2} { %
    (plate1)
    (alpha0)
  } {$i =1, \ldots, N$}; %
 \plate {plate3} { %
    (theta_vi)
    (nu_vi)
  } {$f =1, \ldots, |F|$}; %
  \plate {plate2} { %
    (plate3)
    (alpha0_vi)
  } {$i  = 1, \ldots, N$}; %

\end{tikzpicture}
	\end{center}
  \caption{Directed factor graphs representing hierarchical
    Multinomial-Dirichlet model for related data sets (top panel) and its
    variational approximation (bottom panel). Cat. and Dir. represent
    respectively Categorical and Dirichlet distributions.}
  \label{fig:BN_hier_multi}
\end{figure}
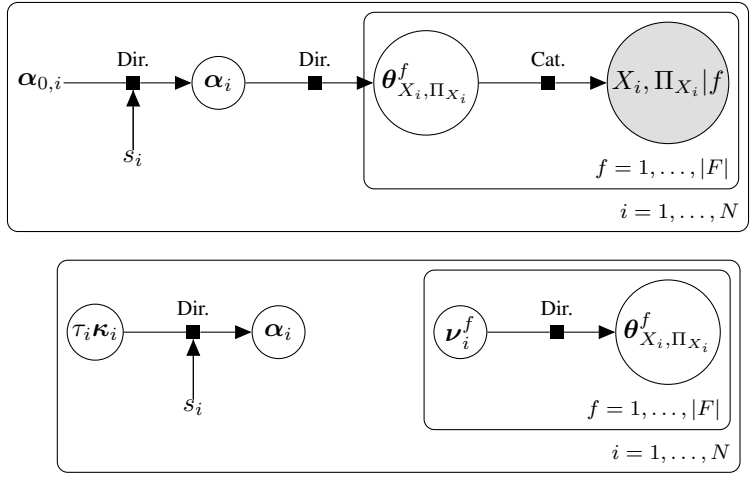

The marginal posterior distribution for $\tXPif$ is not analytically tractable,
as noted in \cite{azzimonti19}. However, the posterior average can be compactly
expressed as:
\begin{align}
  \left[ \htXPif \right]_{jk} =
    \frac{\E[\aijk] + n^f_{ijk}}
         {s_i + n^f_{i}},& &\text{where }
  &n_{i}^f = \sum\nolimits_{jk} n_{ijk}^f.
\label{eq:mean}
\end{align}
$\E[\aijk]$ represents the posterior average of $\aijk$; it cannot be written in
closed form but can be approximated using variational inference \cite{VB1,VB2}.
The resulting $\htXPif$ are data-set-specific but depend on all the available
data via the \emph{partial pooling} \cite{gelman} of the information present in
the $|F|$ related data sets, thanks to the shared $\E[\aijk]$ term. On the one
hand, this produces more reliable estimates for sparse data and for related data
sets with unbalanced sample sizes \cite{casella}. On the other hand, the prior
in \eqref{eq:hier} violates the parameter independence assumption, leading to a
marginal likelihood that does not decompose over parent configurations and that
is not score-equivalent. The prior is specified on $(X_i,\PXi)$, as opposed to
$\XPi$, which is only later computed from the joint distribution. As a result,
the distribution of $(X_i,\PXi \given F)$ is different from the product of the distributions of $(\XPi, F)$ and $(\PXi \given F)$ because
$(X_i,\PXi \given F)$ and $(\PXi \given F)$ are estimated by applying
the hierarchical model separately to two different sets of variables, thus
pooling  the available information differently.

\section{Structure Learning from Related Data Sets}
\label{sec:hBN}
In this section we derive the marginal likelihood score associated with the
hierarchical model in \eqref{eq:hier} to implement structure learning from
related data sets  containing the same variables. As the hierarchical model is
not analytically tractable, we  approximate the associated posterior distribution with the product of two independent distributions by means of variational inference. The approximate
variational model, shown in the bottom panel of Figure~\ref{fig:BN_hier_multi}, is the following:
\begin{align}
  \bai \left|\, s_i, \tau_i, \bkappa_i \right.
    &\sim s_i \cdot \Dirichlet\left(\tau_i \bkappa_i\right), & \nonumber \\
  \left. \tXPif \,\right| \bnu^f_{i}
    &\sim \Dirichlet\left(\bnu^f_{i}\right) & \forallF, \label{eq:hier_vb} 
\end{align}
where $\bnu^f_{i}=\{\nu_{ijk}^f\}$, $\bkappa_i = \{\kijk\}$ with $i = 1,
\ldots N$;  $j = 1, \ldots, |\PXi|$; $k = 1, \ldots, |X_i|$ and $f = 1, \ldots,
|F|$; $\sum_{jk} \kijk = 1$ and $\tau_i\in \mathbb{R}^+$ for $i = 1 \ldots
N$. These parameters are estimated from the available data by minimising the
Kullback-Leibler divergence between the exact posterior distribution $p$ and its
variational approximation $q$, as described in \cite{azzimonti19}. The
algorithm used to estimate the variational parameters is summarised in
\ref{app:estimationalgo}.

Since $F$ is assumed to be the parent of any node in the network and to be
always observed, we treat it as an input variable in a conditional Bayesian
network \cite[Section 5.6]{koller} and we do not explicitly assign it a
distribution. Therefore, the auxiliary variable $F$ will not influence the
score.

The variational model \eqref{eq:hier_vb} is similar to the original hierarchical
MD model \eqref{eq:hier}, but it removes the dependence between $\tXPif$ and
$\bai$ thus making it possible to derive in closed form the variational approximation of the marginal
likelihood $\Prob(\D \given F, \G)$.

\begin{lemma}
Given $|F|$ complete and related data sets $\D=\{\D_f, f = 1, \ldots, |F|\}$,
under the assumption that the related data sets have the same dependence
structure $\G$ and that each local distribution follows the hierarchical MD
\eqref{eq:hier} with positive parameters, the variational approximation of the
marginal likelihood of the data $\Prob(\D \given F, \G)$ is
\begin{equation}
  q(\D \given F, \G)\! =
  \!\prod_{i = 1}^N \prod_{f = 1}^{|F|}\prod_{j = 1}^{|\PXi|}
  \!\!\left[
    \frac{\Gamma(s_i  \hkij)}{\Gamma(s_i \hkij +  n^f_{ij})}
    \prod_{k = 1}^{|X_i|} \frac{\Gamma(s_i \hkijk + n^f_{ijk})}{\Gamma(s_i\hkijk)}
  \right],
\label{eq:bhd}
\end{equation}
where $n^f_{ij}=\textstyle{\sum_k} n^f_{ijk}$, $\hkij=\textstyle{\sum}_k \hkijk$
and $s_i\hkijk$ represents the posterior average of $\aijk$ under the
variational model \eqref{eq:hier_vb}.
\label{lem:marginal}
\end{lemma}

\begin{proof}
Under the hierarchical model \eqref{eq:hier}, the
conditional distribution of $X_i$ given $\PXi = j$ and $F = f$ is a categorical
distribution with parameters $\btheta_{X_i \given j}^f$ whose $k$th element is
$\big[\theta_{X_i,\PXi}^f\big]_{jk}/ \sum_{\tilde{k}}
\big[\theta_{X_i ,\PXi}^f\big]_{j\tilde{k}}$. The distribution of $\btheta_{X_i \given j}^f$ is a Dirichlet distribution with parameter $\balpha_{i \given j}$, whose $k$-th element is $[\balpha_{i \given j}]_k=\aijk /\sum_{\tilde{k}}\alpha_{ij\tilde{k}}$.

Under the variational model \eqref{eq:hier_vb}, the approximate posterior distribution of $\btheta_{X_i \given j}^f$ is a Dirichlet distribution with parameters $\bnu_{ij}^f$ whose
$k$th element is $\nu^f_{ijk}$. 
Since the parameter ${\nu}^f_{ijk}$  is estimated
\emph{a posteriori} as $\widehat{\nu}_{ijk}^f=s_i \hkijk + n_{ijk}^f$ (see
\cite{azzimonti19} for a detailed derivation), we can approximate the prior distribution of $\tXijf$ with  $q(\tXijf \given \G)=\Dirichlet(s_i \widehat{\bkappa}_{ij})$ because of conjugacy.

The variational approximation of the conditional distribution satisfies \emph{independence between
related data sets}.  Moreover, given a data set $F = f$, both \emph{parameter
modularity} and \emph{parameter independence} are satisfied. Thus,
\begin{multline*}
  q(\D \given F, \G) =
    \iint \prod_{f = 1}^{|F|} \prod_{i = 1}^N \prod_{j = 1}^{|\PXi|}
      q(X_i \given \PXi=j, F=f,\tXijf, \balpha_{i \given j}, \G) \\
      q(\tXijf, \balpha_{i \given j} \given \G)
      \dd\tXijf d\balpha_{i \given j}.
\end{multline*}
Thanks to the independence between $\tXijf$ and $\balpha_{i\given j}$ induced by
the variational model and the fact that $\int q(\balpha_{i \given j}
\given \G) \dd\balpha_{i \given j}=1$, we obtain
\begin{align*}
  q(\D \given F, \G)
    = \prod_{f = 1}^{|F|} \prod_{i = 1}^N \prod_{j = 1}^{|\PXi|}
        \int q(X_i \given \PXi = j, F = f, \tXijf, \G)
             q(\tXijf \given \G)
             \dd\tXijf ,
\end{align*}
which has the same form as the marginal likelihood of the classic Multinomial-Dirichlet model but with $s_i \hkijk$ as the parameter of the Dirichlet distribution. 
The approximate marginal likelihood can be thus written as \eqref{eq:bhd}.

\end{proof}

 Note that the marginal likelihood \eqref{eq:bhd} has the same form as
the classic BD score \eqref{eq:bd}, with $\alpha_{ijk}$ replaced by $s_i
\hkijk$, which represents the posterior average of $\alpha_{ijk}$ under the
hierarchical variational model. The posterior average is shared between
different related data sets, thus inducing a pooling effect that makes $\tXPif$
and $\bai$ dependent once more.

From Lemma \ref{lem:marginal}, we define the approximated Bayesian hierarchical
Dirichlet score as
\begin{equation*}
  \BHD(\G, \D \given F) = q(\D \given F, \G).
\end{equation*}
The proposed BHD score can be factorised over the nodes, \textit{i.e.},
\begin{equation*}
  \BHD(\G, \D \given F) = \prod_{i=1}^N \BHD(\XPi, F),
\end{equation*}
and can be used to learn a common structure for all related data sets,
taking into account potential differences in the probabilistic relationships
between variables.

\section{Computational Complexity}
\label{sec:complexity}

Estimating the BHD score in \eqref{eq:bhd} is more complex than estimating the
classic BD score in \eqref{eq:bd} because the latter is available in closed-form
but the former is not. In this section we will assess the computational
complexity of $\BHD(\XPi, F)$ and $\BD(\XPi; \bai)$.

In the case of BD(eu), the score in \eqref{eq:bd} is a closed-form function of
the counts $n_{ijk}$ which are tallied from $\{X_i, \PXi\}$ in
$O(n(1 + N_{\PXi}))$ time, where $N_{\PXi}$ is the dimension of the parent set
$\PXi$. Assuming that each variable takes at most $l$ values, there are
$l^{1 + N_{\PXi}}$ counts. Hence both computing the marginal counts $n_{ij}$ and
multiplying/summing up all the terms in \eqref{eq:bd} take $O\left(l^{1 +
N_{\PXi}}\right)$ time. The overall computational complexity of computing
$\BD\left(\XPi; \bai\right)$ then is
\begin{equation}
  O\left(n(1 + N_{\PXi}) + l^{1 + N_{\PXi}}\right).
  \label{eq:bd-comp}
\end{equation}

As for BHD, the counts $n^f_{ijk}$ are tallied from $\{X_i, \PXi, F\}$ in
$O(n(2 + N_{\PXi}))$ time because of the auxiliary variable $F$. Computing the
marginal counts $n^f_{ij}$ and multiplying/summing up all the terms in
\eqref{eq:bhd} takes $O\left(|F| \, l^{1 + N_{\PXi}}\right)$ time.

The increased complexity of $\BHD(\XPi, F)$, however, comes from the 
algorithm used to estimate the variational parameters $\tau_i$,
$\nu_{ijk}^f$ and $\kijk$. The variational algorithm is derived in \cite{azzimonti19} and is reproduced for convenience as Algorithm
\ref{algo:variational} in \ref{app:estimationalgo}.

The update  of  $\widehat{\nu}_{ijk}^f$ in step \ref{step1} of Algorithm
\ref{algo:variational} takes $O\left(|F| \, l^{1 + N_{\PXi}}\right)$ since there
are as many $\widehat{\nu}^f_{ijk}$ as there are $n_{ijk}^f$.

Each update of $\widehat{\tau}_i$ in step \ref{step3a} requires the computation
of $\partial \mathcal{L} / \partial \tau_i$ and
$\partial^2 \mathcal{L} / \partial^2 \tau_i$: both are closed-form functions
that sum over the indices $j$ and $k$ of $\hkijk$. Updating the
parameter $\widehat{\tau}_i$ by means of \eqref{eq:tau} thus takes
$O\left(l^{1 + N_{\PXi}}\right)$.

Each update of $\hkijk$ in step \ref{step3b} requires the computation of
$\partial \mathcal{L} / \partial \kijk$ and
$\partial^2 \mathcal{L} / \partial^2 \kijk$, which scale respectively as
$|F| l^{1 +N_{\PXi}}$ and as $l^{1 +N_{\PXi}}$. Given the partial derivatives,
the cost of updating the parameter $\hkijk$ by means of \eqref{eq:kappa} scales
as the number of elements, that is, $l^{1 +N_{\PXi}}$. Thus, since there are
 $l^{1 +N_{\PXi}}$ terms $\hkijk$, the computational cost of step \ref{step3b}
is $O\left((|F| + 3) l^{1 + N_{\PXi}}\right)$.

Once we take into account the number of iterations $m_1$ and $m_2$, we have that
the overall computational complexity of computing $\BHD(\XPi, F)$ is
\begin{multline}
  \underbrace{
    O\left(n(2 + N_{\PXi}) + |F| \, l^{1 + N_{\PXi}}\right)
  }_{\text{formula in \eqref{eq:bhd}}} +
  \underbrace{
    O\left(m_1 |F| \, l^{1 + N_{\PXi}}\right)
  }_{\text{step \ref{step1}}} + \\
  +\underbrace{
    O\left(m_1 m_2  \, l^{1 + N_{\PXi}}\right)
  }_{\text{step \ref{step3a}}} +
  \underbrace{
    O\left(m_1 m_2 (|F| \, + 3) l^{1 + N_{\PXi}}\right)
  }_{\text{step \ref{step3b}}} = \\
  =O\left(m_1 m_2 |F| \, l^{1 + N_{\PXi}}\right)
  \label{eq:bhd-comp}
\end{multline}
If we compare \eqref{eq:bhd-comp} with \eqref{eq:bd-comp}, we can see that they
are both linear in $l^{1 + |\PXi|}$, but the former also depends on the number
of related data sets $|F|$.
The computational cost of each step of the iterative procedure for computing BHD is thus comparable to the computational cost associated with learning a network 
by means of BDeu.

\subsection{Empirical evaluation}
In order to confirm the derived computational complexity of BHD, we
evaluate the time needed to compute the BHD score for a single node as different
parameters vary. We consider in particular:
\begin{itemize}
  \item $N_{\PXi} \in \{2, 3, 4, 6, 8, 10, 12, 14, 16\}$, where $N_{\PXi}$
    represents the size of the parent set;
  \item $|F| \in \{5, 10, 20, 40\}$, where $|F|$ represents the number of
    related data sets;
  \item $l = |X_i| \in \{2, 3, 4, 5\}$, where $|X_i|$ represents the number of
    states for each variable.
\end{itemize}
For each parameter combination, we sample 10 different parameter sets with the
following methods:
\begin{itemize}
  \item[\textbf{hier:}] the parameters associated with each of the related data
    sets are sampled from a hierarchical Dirichlet distribution with imaginary
    sample size equal to 10 and with a parameter $\balpha_i$ sampled from a
    Dirichlet distribution with all $\alpha_{0,ijk} = 1$;
  \item[\textbf{iid:}] the parameters associated with each of the related data
    sets are independently sampled from the same Dirichlet distribution with
    imaginary sample size equal to 10 and uniform $\balpha_i$.
\end{itemize}
For each parameter set, we sample $|F|$ related data sets comprising the same
number of observations $n_f \in \{5000, 10000, 20000\}$.

\begin{figure*}[bt]
\begin{center}
  \includegraphics[width=0.9\textwidth]{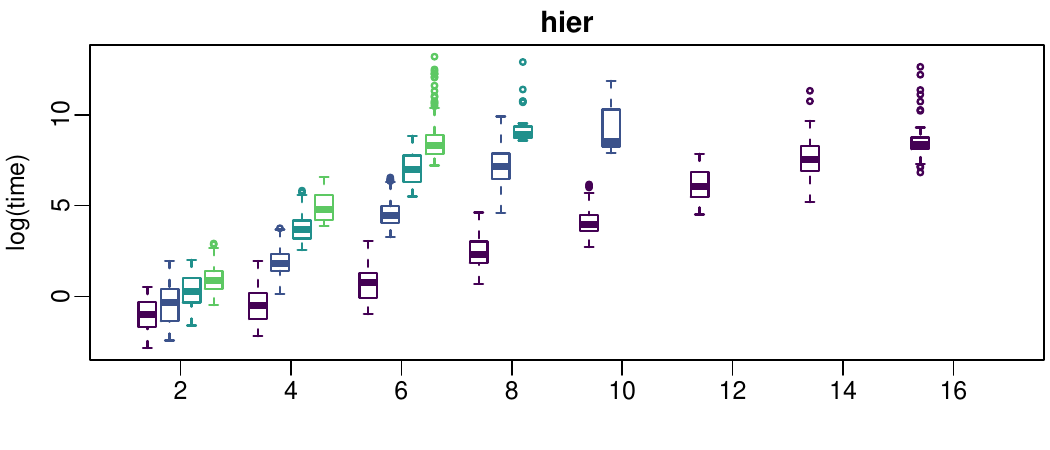}\\
  \includegraphics[width=0.9\textwidth]{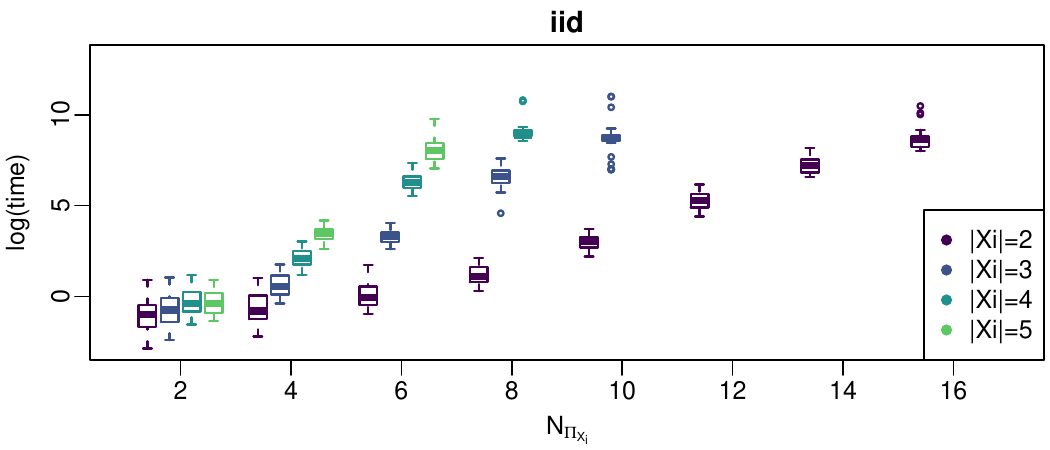}
  \caption{Boxplots of the logarithm of the computational time with different
    dimensions of the parent set $N_{\PXi}$ and number of states $|X_i|$, with
    parameters sampled with the hierarchical (top panel) or i.i.d (bottom panel)
    approach.}
  \label{fig:bigO}
\end{center}
\end{figure*}

Figure \ref{fig:bigO} shows how the logarithm of the computational time varies
as a function of the dimension of the parent set $N_{\PXi}$ for different
values of $|X_i|$. The logarithm of the computational time scales linearly in
the dimension of the parent set for each value of $|X_i|$, with a slope that is
proportional to the value $|X_i|$. A deviation from this behaviour is visible for
small values of $N_{\PXi}$ in the \textbf{iid} case, where the time needed to
estimate the parameters is negligible compared to fixed computational costs
like memory allocation.

The effect of $|F|$ on the computational times is weaker than that of $N_{\PXi}$
and $|X_i|$, while the effect of the number of observations $n_f$ is negligible.

To estimate the effect of $N_{\PXi}$, $|X_i|$ and $|F|$ on the computational
time, we estimated the parameters of the linear model
\begin{equation*}
  \log(\text{Time}) = \beta_0 + \beta_1 \mathbb{I}_{\text{iid}} +
                        \beta_2 \log(|F|) + \beta_3 (1+N_{\PXi}) \log(|X_i|) + \epsilon,
\end{equation*}
corresponding to
\begin{equation*}
  \text{Time} = e^{\beta_0 + \beta_1 \mathbb{I}_{\text{iid}}}
                  |F|^{\beta_2} |X_i|^{\beta_3 \left(1+N_{\PXi}\right)} e^{\epsilon},
\end{equation*}
where $\epsilon$ represents the measurement error. All the parameter estimates
are significantly different from zero (the associated p-values are smaller than
$10^{-15}$) and the model fits the computational times well ($R^2=0.95$).
Moreover, the estimated parameters $\beta_0 = -6.20$, $\beta_1 = -0.84$,
$\beta_2 = 0.71$ and $\beta_3 = 1.15$ are consistent with the theoretical
computational complexity derived in \eqref{eq:bhd-comp}.

\section{Numerical example}
\label{sec:exe}
We consider a simple example to illustrate the steps involved in learning a
network structure with BHD and in estimating the associated parameters.

We consider in this example the set of variables $X_1, \ldots, X_5$, with
$|X_i|=2$ for $i=1, \ldots 5$, and $|F|=2$ related data sets. We assume that the
true underlying structure for both the related data sets is that shown in the
top panel of Figure \ref{fig:demo}, and that the parameters for the two related
data sets are those summarised in Tables \ref{tab:paramX1}-\ref{tab:paramX5} in
 \ref{app:param}. These parameters have been sampled from a hierarchical
Dirichlet distribution with imaginary sample size equal to $10$ and with a
parameter $\balpha$ that is sampled from a Dirichlet distribution with all
$\alpha_{0, ijk} = 1$.

\begin{figure*}[tbhp]
\begin{center}
 \includegraphics[width=0.36\textwidth]{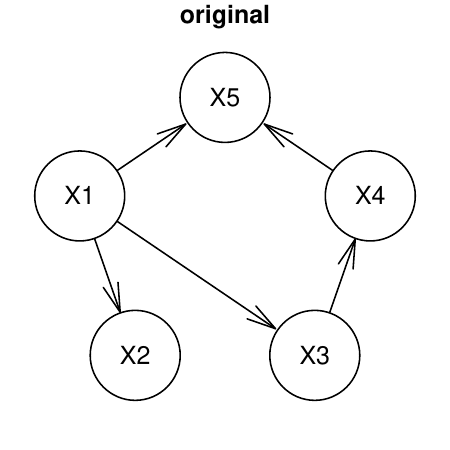}\\
  \includegraphics[width=0.36\textwidth]{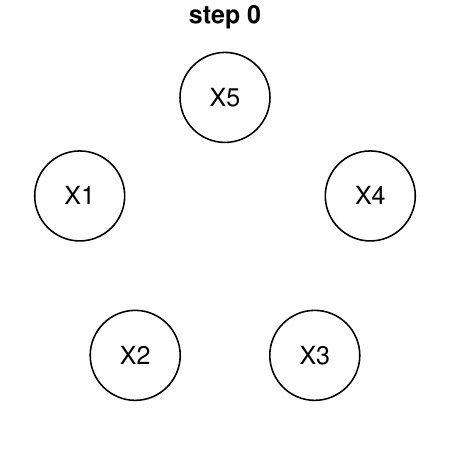}
  \hspace{5pt}
 \includegraphics[width=0.36\textwidth]{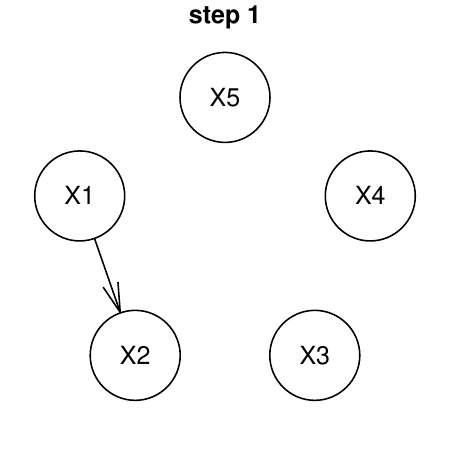}\\
 \includegraphics[width=0.36\textwidth]{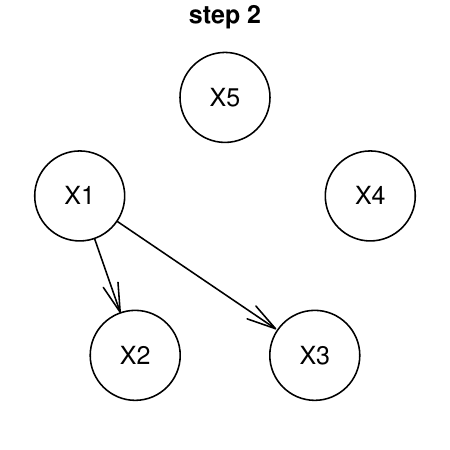}
 \hspace{5pt}
 \includegraphics[width=0.36\textwidth]{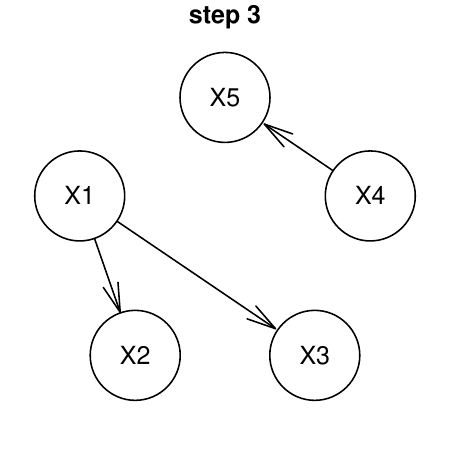}\\
 \includegraphics[width=0.36\textwidth]{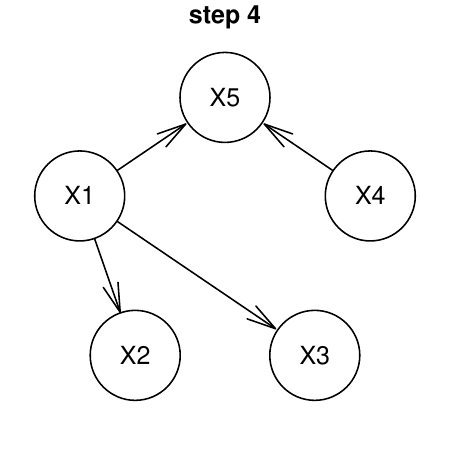}
 \hspace{5pt}
 \includegraphics[width=0.36\textwidth]{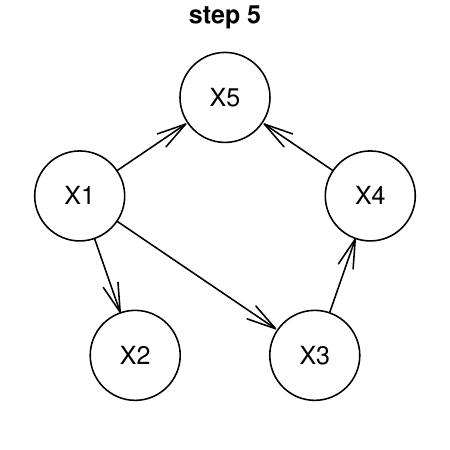}
  \caption{Original underlying network (top central panel) and networks estimated by means of BHD score during the 5 hill-climbing steps. }
  \label{fig:demo}
\end{center}
\end{figure*}

We learn the structure from $|F|$ related data sets, each containing $n_f =
1000$ observations sampled from the true underlying distribution, with the
hill-climbing implementation in bnlearn \cite{jss09} with the BHD score with
imaginary sample $s = 1$.

\begin{table}[bt]
  \begin{center}
    \caption{Scores associated with each node and overall score obtained in the
      numerical example hill-climbing optimisation. Step 0 corresponds to the
      starting empty graph, step 5 corresponds to the estimated network. The scores updated at each step due to arc addition
      are highlighted in bold.}
    \label{tab:scoreH}

    \begin{tabular}{|c|c|c|c|c|c||c|}
      \hline
      \textbf{step} & $\mathbf{X_1}$ & $\mathbf{X_2}$ & $\mathbf{X_3}$ & $\mathbf{X_4}$ & $\mathbf{X_5}$ & \textbf{overall}\\
      \hline
0& -466.205 & -528.450 &-695.655 &-642.506 &-672.477 &-3005.293\\
1& -466.205 & \textbf{-463.584} & -695.655 &-642.506 & -672.477 & -2940.427\\
2& -466.205 & -463.584 & \textbf{-651.734} & -642.506 & -672.477 & -2896.506\\
3& -466.205 &-463.584 & -651.734 &-642.506 &\textbf{-666.052} & -2890.081\\
4& -466.205 &-463.584 &-651.734 &-642.506 & \textbf{-650.698} &-2874.727\\
5& -466.205 &-463.584 &-651.734 &\textbf{-638.011} & -650.698 &-2870.232\\
 \hline
    \end{tabular}
  \end{center}
\end{table}

The true underlying network is recovered after five optimisation steps, shown in
Figure \ref{fig:demo}. The scores associated with each node and the increasing overall
score obtained during the 5 steps of hill-climbing are summarised in Table
\ref{tab:scoreH}. Tables \ref{tab:paramX1}-\ref{tab:paramX5} show the parameters
 estimated by means of the method described in \cite{azzimonti19}, associated with 
 the estimated network. The average absolute error in parameter
estimation is $0.023$. The average absolute error decreases to $0.005$ when the
number of observations increases to $n_f = 10000$, thus showing the consistency
of both the BHD score and the parameter estimation method as the size of the
data set increases. In contrast, even with $n_f = 10000$ both BDeu and BIC are
unable to learn the true underlying network from the pooled data sets due to the
differences between their distributions. Notice, \emph{e.g.}, the parameters associated to the node $X_2$ summarised in Table \ref{tab:paramX2}, which represent different associations between $X_1$ and $X_2$ across the two data sets. BHD is able to take into account different associations across data sets and to properly estimate the associated parameters (see Table \ref{tab:paramX2}). Both BDeu and BIC estimate instead no association between $X_1$ and $X_2$ because they pool the two data sets.

\subsection{Sensitivity to the hyperparameters}
\begin{figure*}[b]
\begin{center}
 \includegraphics[width=0.19\textwidth]{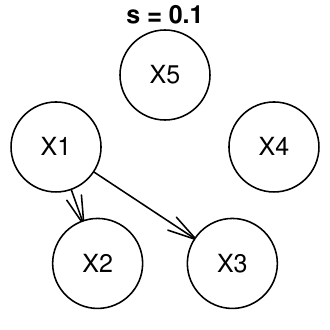}
 \includegraphics[width=0.19\textwidth]{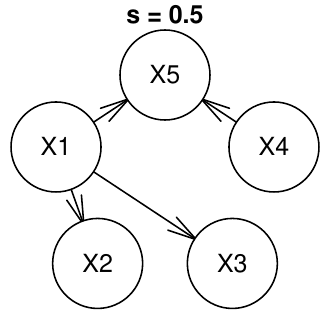}
 \includegraphics[width=0.19\textwidth]{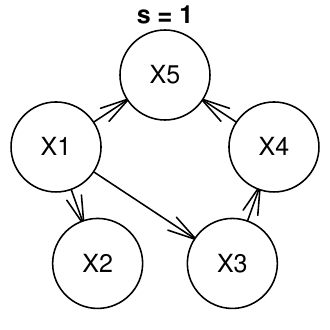}
 \includegraphics[width=0.19\textwidth]{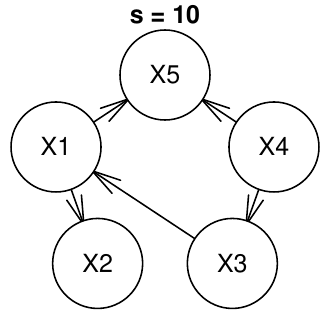}
 \includegraphics[width=0.19\textwidth]{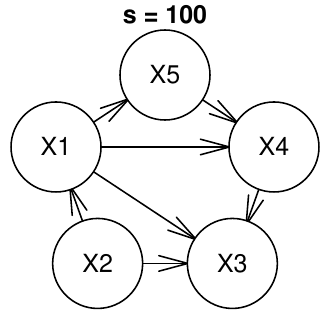}
  \caption{Networks estimated by means of BHD score  with $s \in \{0.1, 0.5, 1, 10, 100\}$ and $s_0=1$. 
  For a given value of $s$, the estimated network does not change as $s_0$ varies.
  }
  \label{fig:demo2}
\end{center}
\end{figure*}
We repeat the experiment for 
different values of the hyperparameters $s$ and  
$s_0$ in the set $\{0.1, 0.5, 1, 10, 100\}$.
The results are summarised in Figure \ref{fig:demo2}.
As $s$ increases the learned networks become more connected, 
similarly to what usually happens with BDeu  (see \cite{silander2007sensitivity} for a more detailed discussion).
However, given a value of $s$, the same graph is learned for all the values of  $s_0$.
In future works, it may be interesting to further study how to choose a suitable value for $s$ for the BHD score or to model it as an hidden random variable with its own prior distribution.

\section{Simulation Studies}
\label{sec:simulations}

We now perform some simulation studies to compare the empirical performance of
BHD to that of BDeu and BIC. For brevity, we will not discuss the results for
BIC in detail since they are fundamentally the same as those for BDeu. We are
interested in structure learning in the following two scenarios:
\begin{enumerate}
  \item[(a)] the true underlying network is the same for all the related data
    sets;
  \item[(b)] the true underlying network is the same for all the related data
    sets, apart from $N_F$ data sets in which $N_A$ randomly selected arcs have
    been removed.
\end{enumerate}
For each scenario, we generate synthetic data following three different models
for the local distributions of each node:
\begin{enumerate}
  \item[\textbf{hier}:] the parameters associated with each of the related data
    sets are sampled from a hierarchical Dirichlet distribution with imaginary
    sample size equal to 10 and with a parameter $\balpha_i$ sampled from
    a Dirichlet distribution with all $\alpha_{0, ijk} = 1$;
  \item[\textbf{iid}:] the parameters associated with each of the related data
    sets are independently sampled from the same Dirichlet distribution with
    imaginary sample size equal to 10 and uniform $\balpha_i$ vector;
  \item[\textbf{id}:] the parameters are identical for all data sets.
\end{enumerate}
The first approach follows the distributional assumptions of the hierarchical
model underlying BHD and may favour the proposed score. The last approach may
favour methods that do not take into account that data may comprise related data
sets, thus pooling all the data and assuming that all observations are generated
from the same distribution. The second approach is a middle ground between the
first and the third, since parameters associated with the related data sets are
different but they are not generated from the hierarchical model.

We then perform structure learning on the simulated data using the hill-climbing
implementation in bnlearn \cite{jss09} with the BHD and BDeu scores, both with
imaginary sample $s = 1$. In the case of BDeu we pool all the available data
from different related data sets. We evaluate the accuracy of network
reconstruction with the Structural Hamming Distance (SHD) \cite{mmhc} between
the estimated and the true underlying structure, True Positive (TP), False
Positive (FP) and False Negative (FN) arcs.

\subsection{Simulation study 1}

The aim of this simulation study is to evaluate the performance of BHD as
different networks parameters vary. Specifically, we consider different number
of nodes, related data sets and states for each variable.

\begin{figure*}[tb]
\begin{center}
  \includegraphics[width=0.45\textwidth]{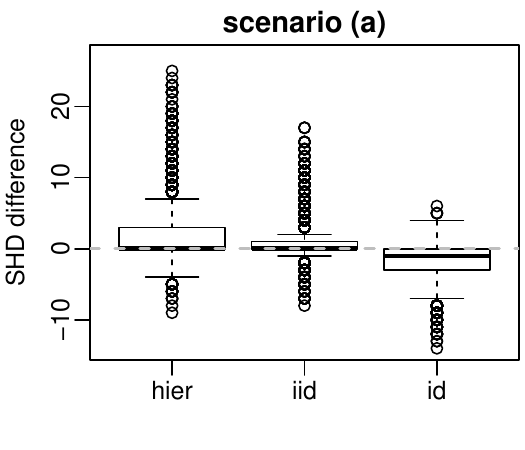}
  \includegraphics[width=0.45\textwidth]{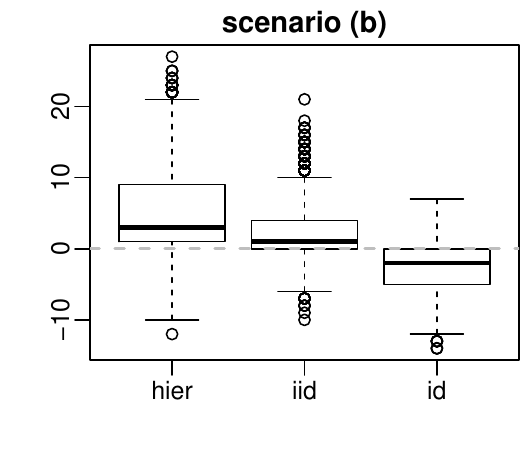}
  \caption{Simulation study 1: Boxplots of SHD difference between BHD and BDeu
    score for scenario (a) (left panel) and (b) (right panel). Positive values
    favour the hierarchical score.}
  \label{fig:shd}
\end{center}
\end{figure*}

\begin{figure*}[p]
\begin{center}
  \includegraphics[width=0.32\textwidth]{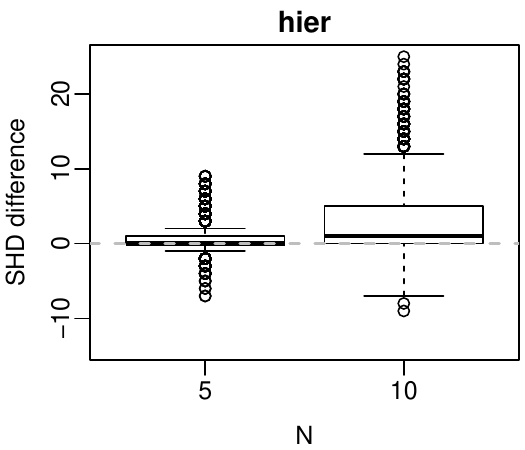}
  \includegraphics[width=0.32\textwidth]{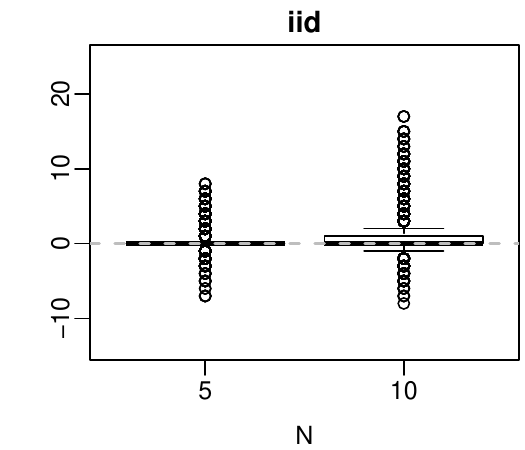}
  \includegraphics[width=0.32\textwidth]{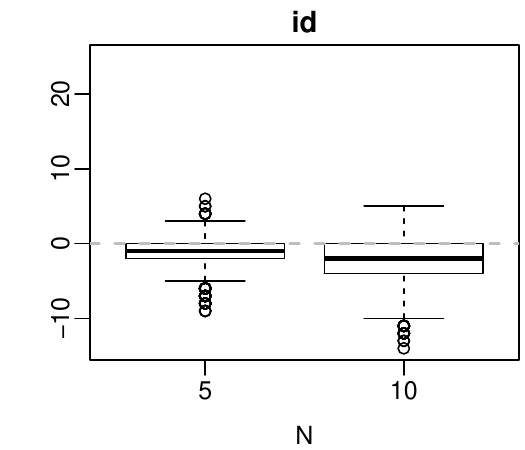} \\
  \vspace{5pt}
  \includegraphics[width=0.32\textwidth]{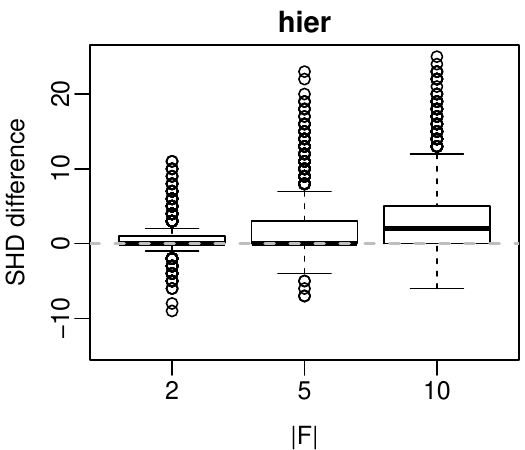}
  \includegraphics[width=0.32\textwidth]{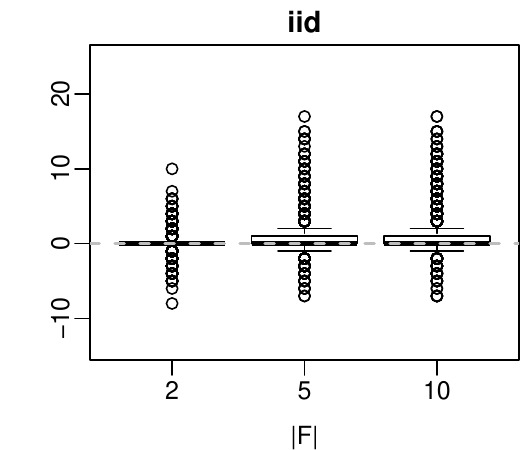}
  \includegraphics[width=0.32\textwidth]{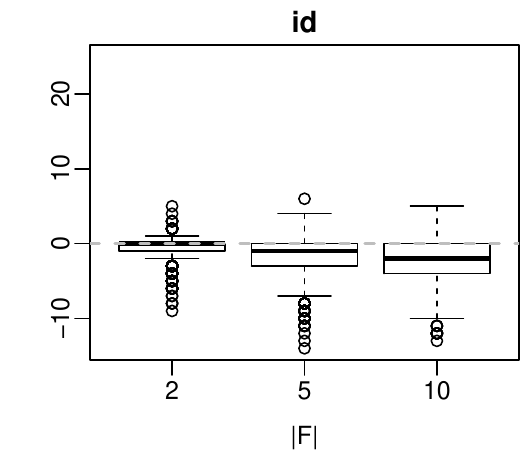} \\
  \vspace{5pt}
  \includegraphics[width=0.32\textwidth]{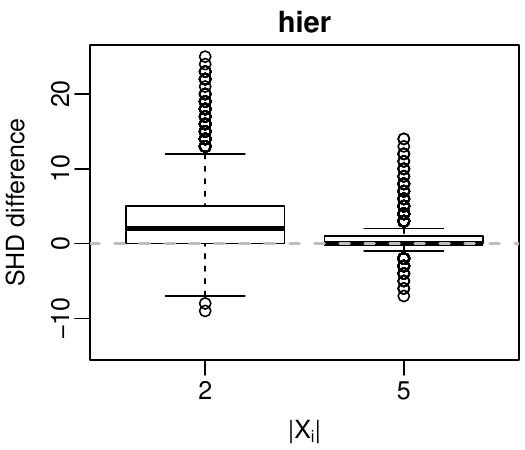}
  \includegraphics[width=0.32\textwidth]{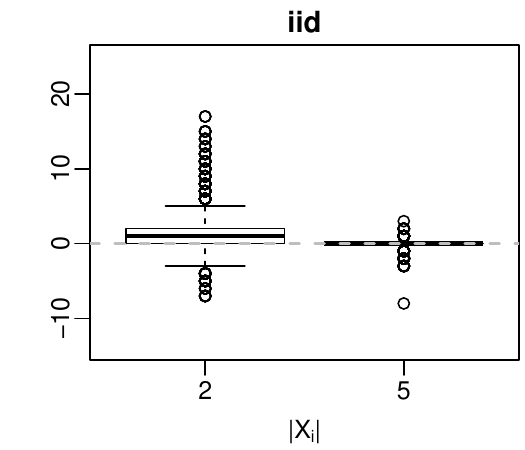}
  \includegraphics[width=0.32\textwidth]{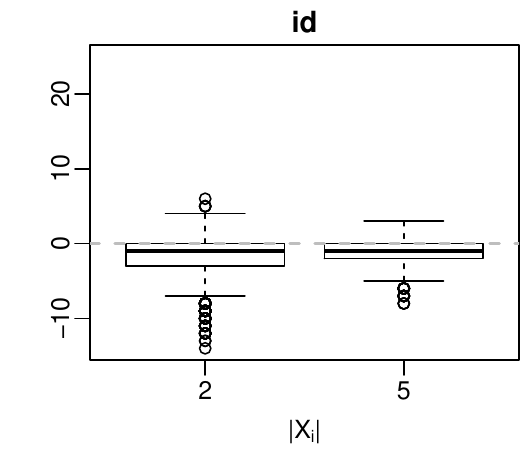} \\
  \vspace{5pt}
  \includegraphics[width=0.32\textwidth]{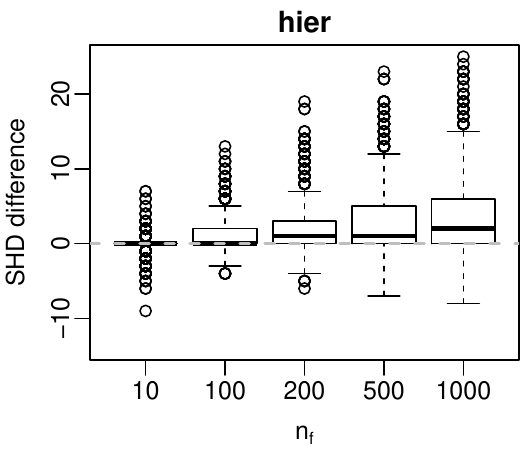}
  \includegraphics[width=0.32\textwidth]{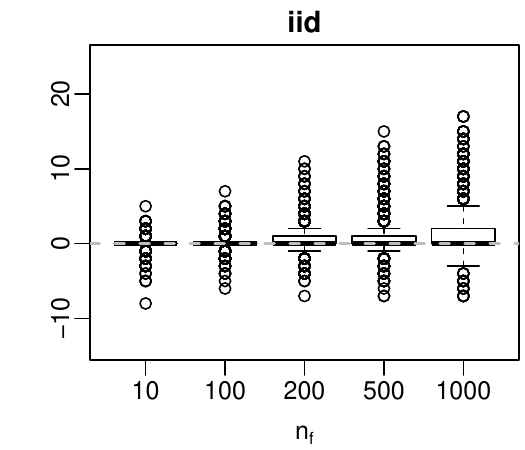}
  \includegraphics[width=0.32\textwidth]{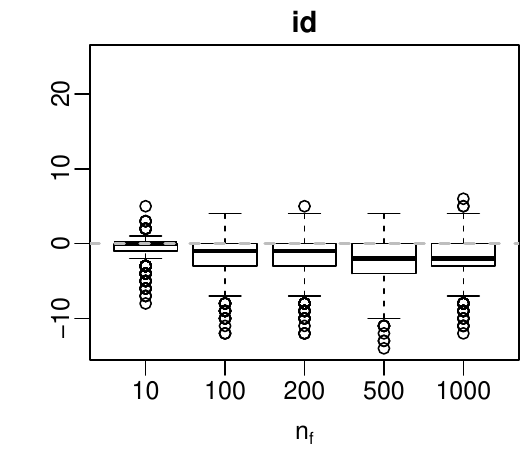}
  \caption{Simulation study 1: Boxplots of SHD difference between BHD and BDeu
    score for scenario (a) (equal structures) with different values of number of
    variables $N$, number of related data sets $|F|$, number of states $|X_i|$ and
    number of observations $n_f$, with parameters sampled with the hierarchical
    (left panels), i.i.d (central panels) or identical distribution (right
    panels) approach. Positive values favour the hierarchical score.}
  \label{fig:shd_param_a}
\end{center}
\end{figure*}

We first sample 3 network structures for each of three different levels of
sparsity, such that they contain $\{1, 1.2, 1.5\}\cdot N$ arcs, and each
combination of:
\begin{itemize}
  \item $N \in \left\{5, 10\right\}$, where $N$ represents the number of nodes;
  \item $|F| \in \left\{2, 5, 10\right\}$, where $|F|$ represents the number of
    related data sets;
  \item $|X_{i}| \in \left\{2, 5\right\}$, where $|X_{i}|$ represents the
    number of states for each variable.
\end{itemize}
Then, for both scenario (a) and (b), we replicate the same structure for all the
$|F|$ related data sets. In scenario (b), for each of the $N_F$ data sets
differing from the others, we randomly remove $N_A$ arcs from the network, with
$N_F\in\{1,2\}$ and $N_A\in\{1,2\}$. Thus, in scenario (b) we deal with $N_F$
structures that differ from one another and from the main structure by $N_A$
arcs.

Once the network structures have been generated, we sample 10 different
parameter sets for each of \textbf{hier}, \textbf{iid} and \textbf{id}. Then,
for each of these parameter sets, we sample 
$|F|$ related data sets, each containing $n_f \in \{10, 100, 200, 500, 1000\}$
observations.

\begin{figure*}[t]
\begin{center}
  \includegraphics[width=0.32\textwidth]{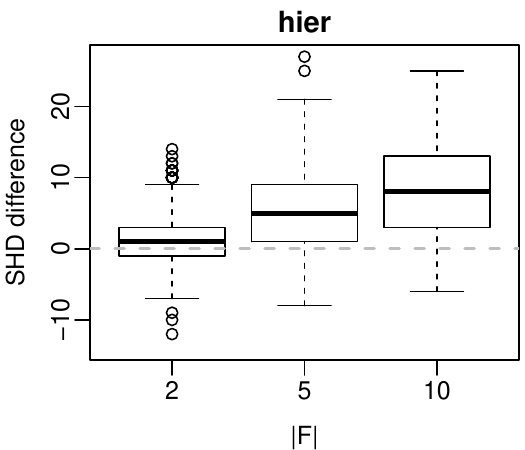}
  \includegraphics[width=0.32\textwidth]{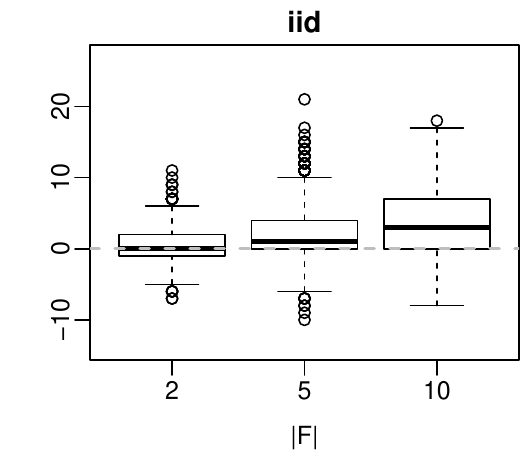}
  \includegraphics[width=0.32\textwidth]{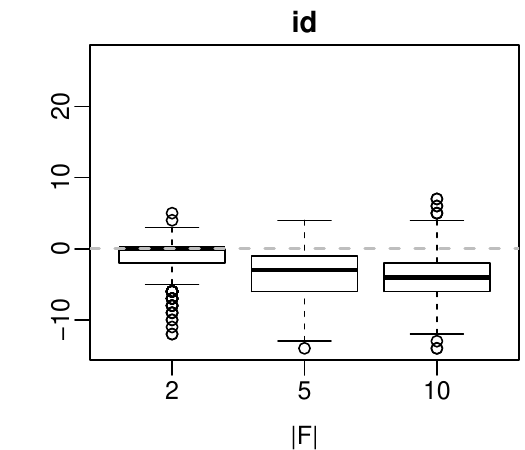} \\
  \vspace{5pt}
  \includegraphics[width=0.32\textwidth]{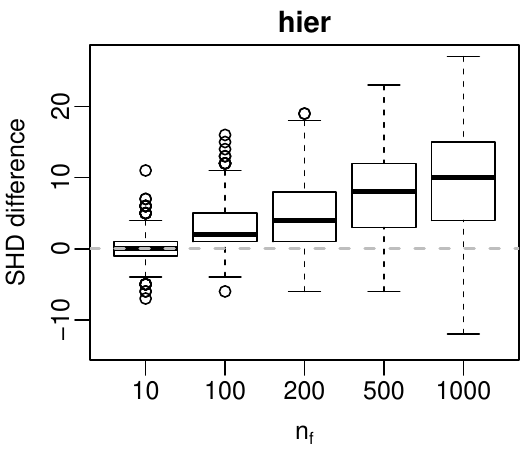}
  \includegraphics[width=0.32\textwidth]{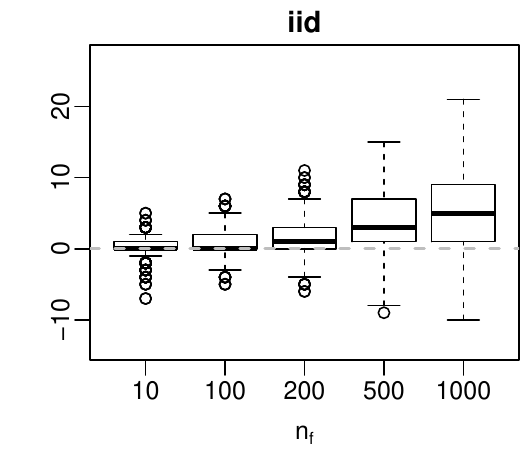}
  \includegraphics[width=0.32\textwidth]{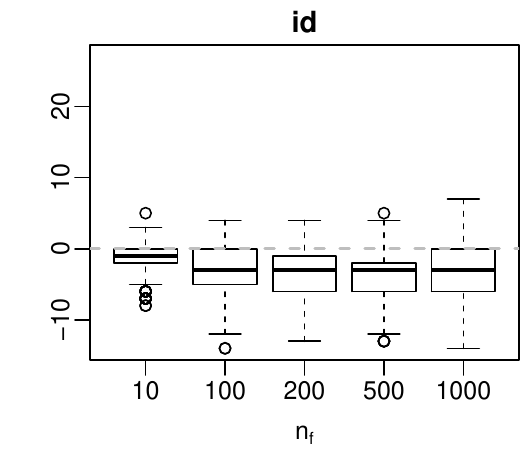}
  \caption{Simulation study 1: Boxplots of SHD difference between BHD and BDeu
    score for scenario (b) (different structures) with different values of
    number of related data sets $|F|$ and number of observations $n_f$, with
    parameters sampled with the hierarchical (left panels), i.i.d (central
    panels) or identical distribution (right panels) approach. Positive values
    favour the hierarchical score.}
  \label{fig:shd_param_b}
\end{center}
\end{figure*}

\begin{figure*}[t]
\begin{center}
  \includegraphics[width=0.32\textwidth]{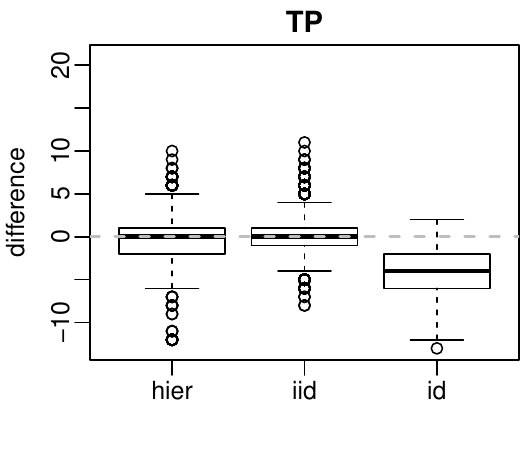}
  \includegraphics[width=0.32\textwidth]{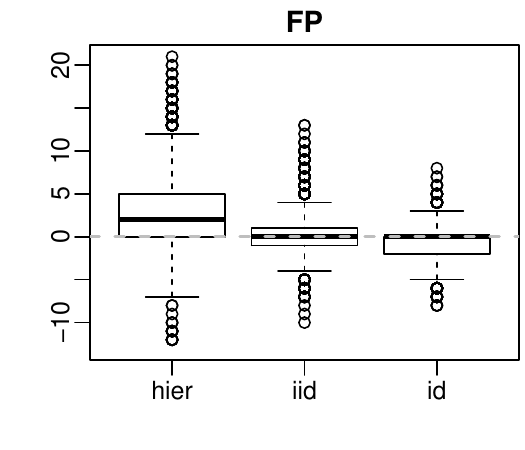}
  \includegraphics[width=0.32\textwidth]{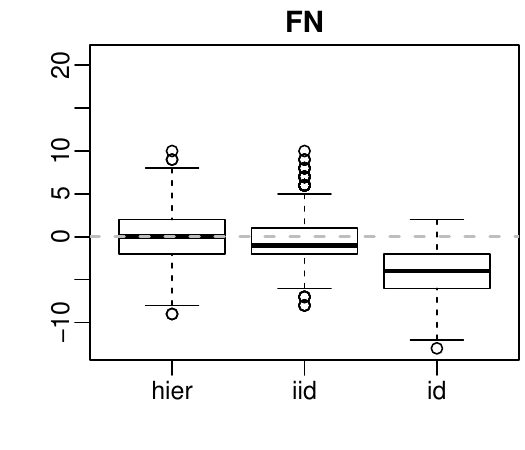}
  \caption{Simulation study 1: Boxplots of TP (left panel), FP (central panel)
    and FN (right panel) difference between BHD and BDeu score for scenario (b)
    (different structures). Positive values always favour the hierarchical
    score.}
  \label{fig:TPFPFN}
\end{center}
\end{figure*}

The difference between BDeu and BHD in terms of SHD for scenarios (a) and (b) is
shown in Figure \ref{fig:shd}, respectively in the left and right panel.
Positive values favour the proposed BHD score. When parameters are sampled from
a hierarchical distribution (\textbf{hier}), BHD outperforms BDeu in both
scenarios, with a larger improvement in scenario (b). In the \textbf{iid} case,
BHD is competitive with BDeu when the underlying network structures are
homogeneous, and it outperforms BDeu when the underlying network structures are
different. On the other hand, in the \textbf{id} case BDeu has better accuracy
than BHD because it correctly assumes that all the data are generated form the
same distribution, while BHD has a large number of redundant parameters that
would model the non-existing related data sets.

Figure \ref{fig:shd_param_a} shows how the difference in SHD between BHD and
BDeu varies for different simulation parameters in scenario (a). Specifically,
the differences between BHD and BDeu (positive for \textbf{hier} and
\textbf{iid}, negative for \textbf{id}) become increasingly large in magnitude
as the number of variables $N$ or the number of related data sets $|F|$
increase. On the other hand, the differences between BHD and BDeu gradually
decrease as the number of states $|X_i|$ increases. As for the sample size, BHD
increasingly outperforms BDeu in both \textbf{hier} and \textbf{iid} as $n_f$
increases. In the \textbf{id} case we expect the two scores to be asymptotically
equivalent, but the values we consider for $n_f$ are not large enough to clearly
show it empirically.

Figure \ref{fig:shd_param_b} shows  the relationship between the difference in
SHD and some key simulation parameters in scenario (b). The effect of both the
number of related data sets $|F|$ and the number of observations $n_f$ is more marked than in scenario (a). For the same $|F|$ and $n_f$, BHD outperforms BDeu
by a larger margin when some network structures are different (scenario (b))
compared to when they are all identical (scenario (a)).

Figure \ref{fig:TPFPFN} shows the difference in TP (left), FP (center) and FN
(right panel) between BDeu and BHD for scenario (b). Positive values favour the
proposed BHD score. While the two methods perform similarly in terms of TP and
FN, BHD outperforms BDeu in terms of FP in the \textbf{hier} case. The
structures learned by BHD are thus sparser and more interpretable than those
learned by BDeu.

We also perform some experiments with different values $s \in \{1,2,5,10\}$ of
the imaginary sample size. 
As $s$
increases, BHD achieves marginally lower SHDs. However, its average SHD is not
significantly different from that of BDeu for the same value of $s$.

\subsection{Simulation study 2}

Given the results of the first simulation study, we now focus on the effect of
specific parameters on the performance of BHD. In particular, the aim of this
simulation study is to evaluate the behaviour of  the proposed score  as the number of
related data sets increases.

Similarly to the first simulation study, we sample one network structure for each
of three different levels of sparsity ($\{1, 1.2, 1.5\}\cdot N$ arcs as before) and
each of $|F| \in \left\{2, 5, 10, 25, 50, 100\right\}$ related data sets. We
treat both the number of nodes ($N = 10$) and the number of states for each
variable ($|X_{i}| = 2$) as fixed.

Then, for both scenario (a) and (b), we replicate the same structure for all the
$|F|$ related data sets. In scenario (b), for each of the $N_F$ data sets
differing from the others, we randomly remove $N_A$ arcs from the network, with
$N_F\in\{1, 2, |F|\}$ and $N_A \sim \text{Bin}(N_T,p)$, where $N_T$ is the total
number of arcs and $p\in\{0.01, 0.1\}$. Thus, in scenario (b) we deal with $N_F$
structures that differ from one another and from the main structure by $N_A$
arcs.

Once the network structures have been generated, we sample 10 different
parameter sets for each structure and for each of \textbf{hier} and
\textbf{iid}; for each of these parameter sets, we sample $|F|$ related
data sets composed of $n_f \in \{10, 100, 200, 500, 1000, 2000,$ $5000,10000\}$
observations each. In this simulation study we disregard the \textbf{id} case and
we focus instead on the more interesting \textbf{hier} and \textbf{iid} cases.

Figures \ref{fig:Zdim} and \ref{fig:Zdim_n} show how the  SHD difference between
BHD and BDeu varies as the number of related data sets $|F|$ and the number of
observations $n_f$ increase. Results obtained in scenario (a) and (b) are
presented together for brevity. Moreover, we did not notice any practical
difference as the number of removed arcs $N_A$ or the number of networks with a
reduced number of arcs $N_F$ vary.

For small values of $|F|$ the improvement of BHD with respect to BDeu increases
with $|F|$, while for large values of $|F|$ the difference between BHD and BDeu
reaches a plateau. As expected, the gain is much larger in \textbf{hier}. In
both \textbf{hier} and \textbf{iid}  BHD increasingly outperforms BDeu  as $n_f$
increases. The difference between the two scores is particularly clear for
large values of $n_f$.

\begin{figure*}[t]
\begin{center}
  \includegraphics[width=0.45\textwidth]{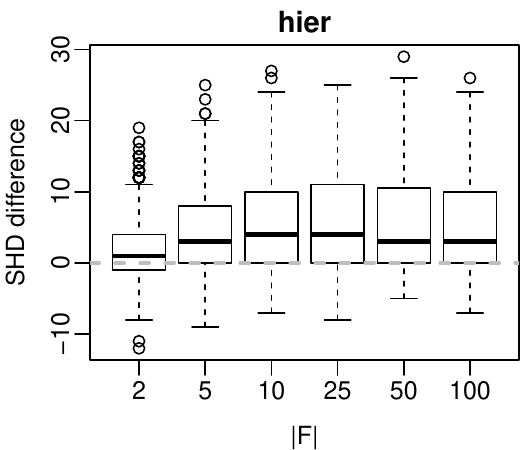}
  \includegraphics[width=0.45\textwidth]{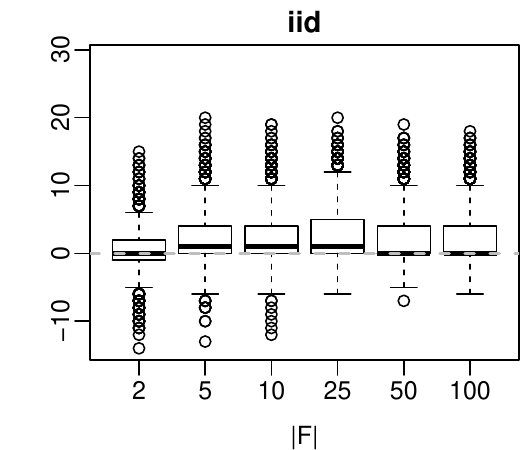}\\
  \caption{Simulation study 2: Boxplots of SHD difference between BHD and BDeu
    score with different values of number of related data sets $|F|$, with
    parameters sampled with the hierarchical (left panel) or  i.i.d (right
    panel) approach. Positive values favour the hierarchical score.}
  \label{fig:Zdim}
\end{center}
\end{figure*}

\begin{figure*}[t]
\begin{center}
  \includegraphics[width=0.9\textwidth]{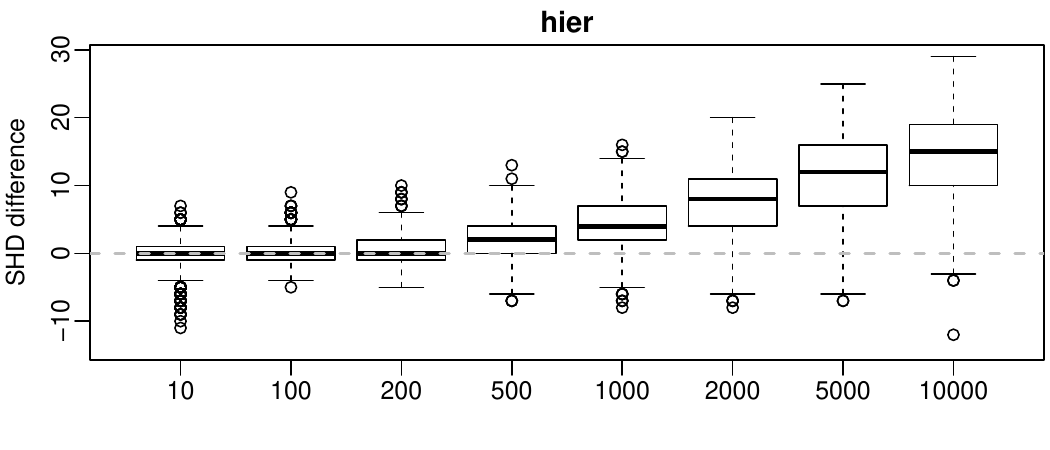}\\
  \includegraphics[width=0.9\textwidth]{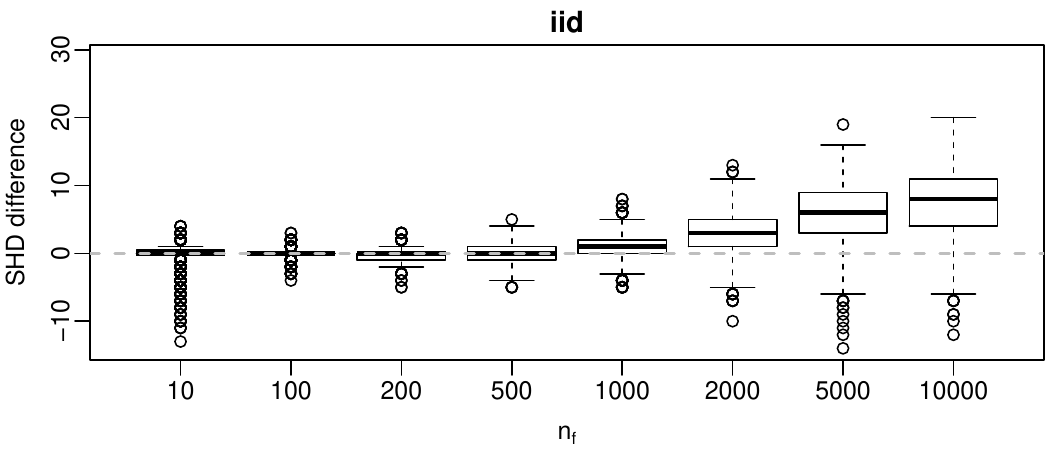}
  \caption{Simulation study 2: Boxplots of SHD difference between BHD and BDeu
    score with different values of number of observations $n_f$, with parameters
    sampled with the hierarchical (top panel) or  i.i.d (bottom panel) approach.
    Positive values favour the hierarchical score.}
  \label{fig:Zdim_n}
\end{center}
\end{figure*}

The improvement of BHD with respect to BDeu in terms of SHD is the result of a
lower number of FP arcs, as in the first simulation study. However, in this
study BHD outperforms BDeu also in terms of TP and FN for large values of $n_f$.

\subsection{Simulation study 3}

Following up from the second simulation study, we now evaluate the performance
of BHD as the number of nodes varies.

We sample a network structure for each of three different levels of sparsity
(the same as in the first two studies) and each of
$N \in \left\{5, 10, 25, 50, 100\right\}$ nodes. We treat the
number of related data sets ($|F| = 10$) and the number of states for each
variable ($|X_{i}| = 2$) as fixed.

The simulations are performed following the same steps as in the previous
simulation study. For both scenario (a) and (b), we replicate the same structure
for all the $|F|$ related data sets. We then we randomly remove $N_A$ arcs from
the networks for each of the $N_F$ data sets differing from the others, as in simulation study 2; then we
sample 10 parameter sets and, for each of them, $|F|$ related data sets with
$n_f \in \{10, 100, 200,$ $500, 1000, 2000,$ $5000,10000\}$ observations.

Figure \ref{fig:k} shows the SHD difference between BHD and BDeu as a function
of the number of nodes $N$. Results obtained in scenario (a) and (b) are
presented together for brevity as in simulation study 2. The boxplots show that
the bigger the number of nodes $N$, the larger the improvement of BHD with
respect to BDeu. As expected, the gain is much larger in \textbf{hier} compared
to \textbf{iid} for the same $N$. 

\begin{figure*}[tb]
\begin{center}
  \includegraphics[width=0.45\textwidth]{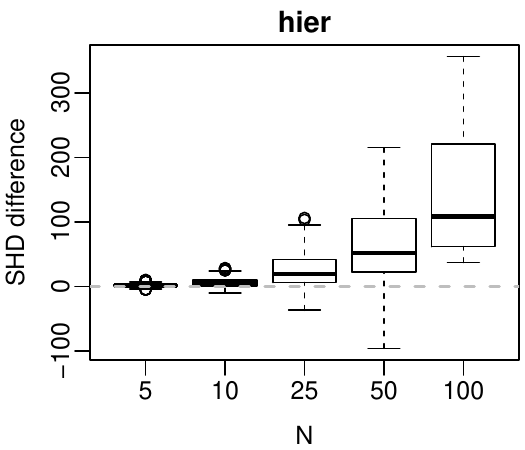}
  \includegraphics[width=0.45\textwidth]{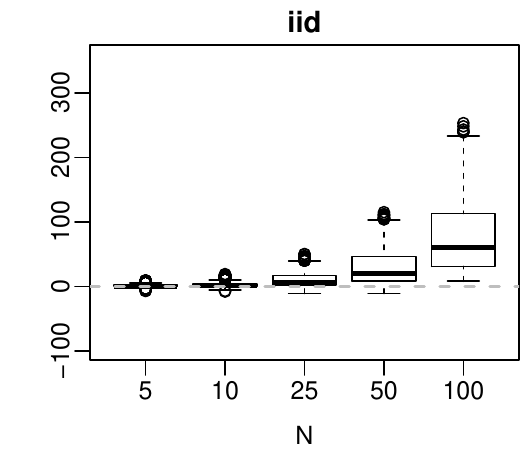}\\
  \caption{Simulation study 3: Boxplots of SHD difference between BHD and BDeu
    score with different values of number of variables $N$, with parameters
    sampled with the hierarchical (left panel) or  i.i.d (right panel) approach.
    Positive values favour the hierarchical score.}
  \label{fig:k}
\end{center}
\end{figure*}

As in the previous simulation study, the improvement of BHD with respect to BDeu
can be attributed to a lower number of FP arcs. However, in this simulation
study BHD outperforms BDeu also in terms of TP and FN for large values of $N$.

\subsection{Simulation study 4}

\begin{figure*}[tb]
\begin{center}
  \includegraphics[width=0.9\textwidth]{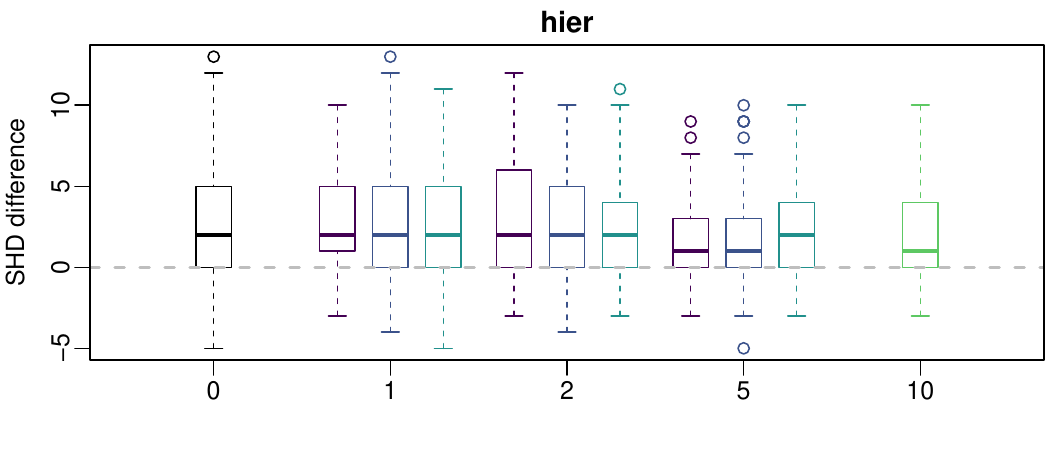}\\
  \includegraphics[width=0.9\textwidth]{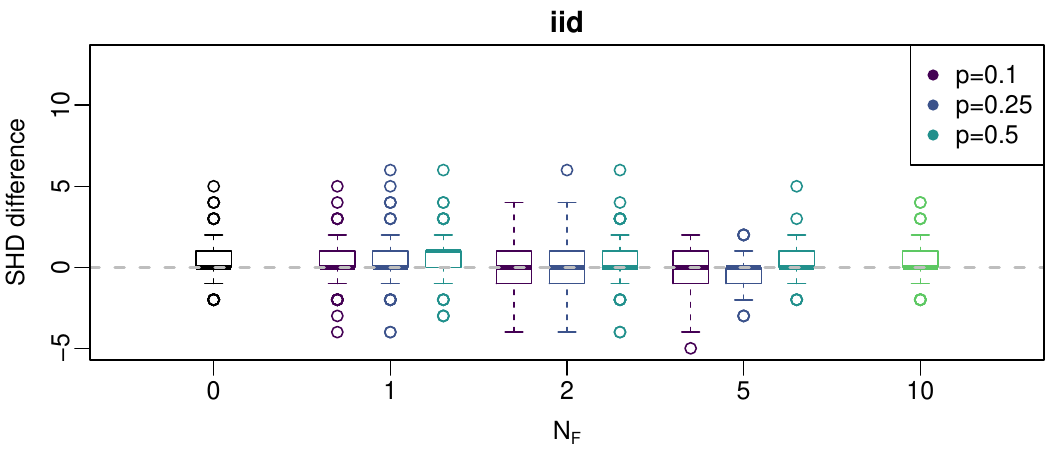}
  \caption{Simulation study 4: Boxplots of SHD difference between BHD and BDeu
    score with different number of sub-sampled data sets $N_F$ and
    proportions of sub-sampling $p$, with parameters sampled with the
    hierarchical (top panel) or  i.i.d (bottom panel) approach. $N_F = 0$
    corresponds to the case where all the data sets have the same number of
    observations,  $N_F = 10$ corresponds to the case where the samples have
    dimensions $(0.25 n_f, 0.25 n_f, 0.5n_f, 0.5n_f, 0.75n_f, 0.75n_f, 0.75n_f,
    n_f, n_f, n_f)$ . Positive values favour the hierarchical score.}
  \label{fig:NF}
\end{center}
\end{figure*}

The aim of this simulation study is to evaluate the performance of BHD when each
related data set is composed by a different number of observations. The
simulation proceeds as in the previous studies, while treating the number of
related data sets ($|F| = 10$), the number of nodes ($N = 10$) and the number of
states for each variable ($|X_i| = 2$) as fixed.

Only for scenario (a), we replicate the same structure for all the $|F|$ related
data sets, we sample 10 different parameter sets for each structure and for each
of \textbf{hier} and \textbf{iid}. Then, for each of these parameter sets, we
sample $|F|$ related data sets, each of them composed of a number of
observations $n_f  \in \{100,500,1000\}$ which is the same for all the $|F|$ related data sets apart
from $N_F$ data sets that are reduced to $p \cdot n_f$ observations, where
$p \in \{0.1, 0.25, 0.5\}$.

Furthermore, we consider an additional case with a different number of
observations for each related data set. Specifically we consider for the 10
related data sets a number of samples equal to  $(0.25 n_f, 0.25 n_f, 0.5n_f,
0.5n_f, 0.75n_f, 0.75n_f, 0.75n_f,$ $n_f, n_f, n_f)$. This composite case will
be identified by means of $N_F = 10$. We compare all these cases using the
standard case where all the related data sets are composed by the same
number of observations $n_f$ as a baseline. This case will be identified by
means of $N_F = 0$.

Figure \ref{fig:NF} shows the SHD difference between BHD and BDeu for different
values of $N_F$ and $p$. The performance of BHD in all the sub-sampled scenarios
is similar to the standard case ($N_F = 0$). Also in the composite scenario,
with a different number of observations for each related data set ($N_F = 10$),
the performance of BHD do not decrease significantly  with respect to the
equal-number-of-observation scenario. Analogously to the previous simulation
studies, the gain is much larger in \textbf{hier}.

\section{Conclusions and future work}
\label{sec:concl}

In this work we propose a new Bayesian score, BHD, to learn a common BN
structure from related data sets. BHD assumes that their joint distribution in
each node of the network follows a mixture of Dirichlet distributions, thus
pooling information between the data sets. The joint distribution in each node
is approximated by means of a variational method. We found that the resulting
computational complexity is linear in the number of related data sets, and that
otherwise it is in the same class as BDeu. We showed with a comprehensive set of
simulation studies that BHD outperforms both BDeu and BIC when applied to data
that comprise related data sets; and that it has comparable performance to BDeu
and BIC when the data are a single, homogeneous set of observations. Moreover,
the larger is the number of the nodes in the network or the larger is the number of
observations, the larger is the improvement with respect to both BDeu and BIC.

Learning a common BN structure with BHD builds on and complements our previous
work on parameter learning from related data sets, described in
\cite{azzimonti19}. We can use the latter to learn the parameters associated
with a network structure learned using BHD, thus obtaining different BNs (one
for each related data set) with the same structure and related parameters, as shown in the numerical example.
Combining the two approaches may increase the performance of the BN models such
as BN classifiers when dealing with related data sets. 
Future applications of the combined approach include, \emph{e.g.}, meta-analysis studies, which aim at combining information from several data sources \cite[Sec. 5.6]{gelman}.

The assumptions underlying BHD can be relaxed in several ways to extend its
applicability to more complex scenarios. For instance, relaxing the
assumption that related data sets share the same dependence structure may allow
to detect independencies that hold only in certain contexts, as in
\cite{boutilier}. Such context-specific independences would be directly
modelled by learning different but related network structures for each data set.
Another interesting development would be to derive a conditional independence
test from BHD to learn BNs from related data sets with constraint-based
algorithm similarly to, \textit{e.g.}, \cite{tillman}.

\appendix

\section{Estimation of variational parameters}
\label{app:estimationalgo}

For each node $X_i$, with $i=1,\ldots,N$, we estimate the variational
parameters of model \eqref{eq:hier_vb} by maximising 
a lower bound
of the evidence lower bound (ELBO) for the marginal log-likelihood for the node
$X_i$. The derivation of  the bound and the algorithm used to estimate the variational parameters are described in
detail in \cite{azzimonti19}. Here we summarise the parameter estimation method in Algorithm 
\ref{algo:variational} and we adapt it to match the notation used  in this paper.

\begin{algorithm}[t]
\caption{Variational Estimator}
\label{algo:variational}

\vspace{2mm}
For each $i=1,\ldots,N$, while $\text{iter}_1 < m_1$ and $\mathit{tol_1} > t_1$:
\begin{enumerate}
  \item Update $\widehat{\bnu}_{i}$ by means of \eqref{eq:nu}
    \label{step1}
  \item Fix the starting values of $\widehat{\tau}_i$,
    $\widehat{\bkappa}_{i}$. \label{step2}
  \item While $\text{iter}_2 < m_2$ and $\mathit{tol_2} > t_2$: \label{step3}
    \begin{enumerate}
      \item update $\widehat{\tau}_i$ given $\widehat{\bkappa}_{i}$
        and $\widehat{\bnu}_{i}$ by means of \eqref{eq:tau};
        \label{step3a}
      \item update $\widehat{\bkappa}_{i}$ given $\widehat{\tau}_i$
        and $\widehat{\bnu}_{i}$ by means of \eqref{eq:kappa};
        \label{step3b}
      \item increase the iterator $\text{iter}_2$. \label{step3c}
    \end{enumerate}
  \item Update $\widehat{\tau}_i$, $\widehat{\bkappa}_{i}$ with the
    values estimated at the end of step \ref{step3}. \label{step4}
  \item Increase the iterator $\text{iter}_1$. \label{step5}
\end{enumerate}
\end{algorithm}

The lower bound
of the ELBO for the node
$X_i$ is the the functional ${\mathcal{L}}_i$, which is defined as
\begin{align*}
  {\mathcal{L}}_i =&
\sum_{f=1}^{|F|} \sum_{k=1}^{|X_i|} \sum_{j=1}^{|\PXi|} (n_{ijk}^f-\nu_{ijk}^f+s_i\kijk) (\psi(\nu_{ijk}^f)-\psi(\nu_{i \cdot\cdot}^f)) + \\
&+\sum_{f=1}^{|F|} \sum_{k=1}^{|X_i|} \sum_{j=1}^{|\PXi|} \log\Gamma(\nu_{ijk}^f) -|F| \sum_{k=1}^{|X_i|} \sum_{j=1}^{|\PXi|}\log\Gamma(s_i\kijk) + \\
&+|F| \sum_{k=1}^{|X_i|} \sum_{j=1}^{|\PXi|}(s_i\kijk - 1)(\log(\kijk)-\psi(\tau_i\kijk)+\psi(\tau_i)) + \\
&+ \sum_{k=1}^{|X_i|} \sum_{j=1}^{|\PXi|}([\balpha_0]_{ijk}-\tau_i \kijk)(\psi(\tau_i\kijk)-\psi(\tau_i))  + \\
&+\sum_{k=1}^{|X_i|} \sum_{j=1}^{|\PXi|} \log\Gamma(\tau_i\kijk) -\sum_{k=1}^{|X_i|} \sum_{j=1}^{|\PXi|} \log\Gamma([\balpha_0]_{ijk}) -\sum_{f=1}^{|F|}\log\Gamma(\nu_{i \cdot\cdot}^f)+\\
&+|F|\log\Gamma(s_i)- \frac{s_i}{\tau_i}|F|\left(|X_i||\PXi|-1\right)+\log\Gamma\left(s_0\right) - \log\Gamma(\tau_i),
\end{align*}
where $\nu_{i \cdot\cdot}^f=\sum_{k=1}^{|X_i|} \sum_{j=1}^{|\PXi|} \nu_{ijk}^f$ and $\psi(\cdot)$ is the digamma function, the derivative of the $\log\Gamma(\cdot)$ function.

For each data set $F=f$, with $f=1,\ldots,|F|$, we can first estimate the
quantity $\nu_{ijk}^f$, associated with the configuration $k$ of $X_i$ and $j$
of $\PXi$, by maximising $\mathcal{L}_i$ with respect to  $\nu_{ijk}^f$ and by assuming
$\kijk$ to be fixed. By setting the partial derivative
of $\mathcal{L}_i$ with respect to $\nu_{ijk}^f$ to zero, we obtain:
\begin{equation}
  \hat{\nu}_{ijk}^f = n_{ijk}^f + s_i \kijk.
\label{eq:nu}
\end{equation}

We can then estimate $\hat{\tau}_i$ and $\hat{\kappa}_{ijk}$ given the value of 
$\nu_{ijk}^f$ by means of a fixed-point method, which alternates the
optimisation of $\mathcal{L}_i$ with respect to $\tau_i$ and $\kijk$. Since no 
analytical solution is available, we perform this optimisation by means of a Newton
algorithm.

We obtain the Newton update for $\tau_i$ by treating $\bkappa_i$ as fixed to
a vector whose elements are the
quantities $\kijk$ with $j = 1, \ldots, |\PXi|$ and $k = 1, \ldots, |X_i|$. If
we define
\begin{align*}
  &g_{\tau_i}(\tau_i,\boldsymbol{\kappa}_i)
    = {\partial\mathcal{L}_i}/{\partial \tau_i} = \frac{s_i}{\tau_i^2}|F|\left(|X_i||\PXi| - 1\right)+\\
    &+ \sum_{k=1}^{|X_i|} \sum_{j=1}^{|\PXi|}
         (\psi'(\tau_i\kijk)\kijk - \psi'(\tau_i))
         ([\balpha_0]_{ijk} - \tau_i \kijk
           - |F|(s_i\kijk - 1))
\end{align*}
and
\begin{align*}
  &h_{\tau_i}(\tau_i, \bkappa_i)
    = {\partial^2\mathcal{L}_i}/{\partial \tau_i^2} = - \frac{2s_i}{\tau_i^3}|F|\left(|X_i||\PXi| - 1\right) - \sum_{k=1}^{|X_i|} \sum_{j=1}^{|\PXi|}
         \psi'(\tau_i \kijk) \kappa^2_{ijk} +
         \\
    & +\psi'(\tau_i)+\sum_{k=1}^{|X_i|} \sum_{j=1}^{|\PXi|}
         (\psi''(\tau_i \kijk)\kappa^2_{ijk}-\psi''(\tau_i))
         ([\balpha_0]_{ijk} - \tau_i \kijk
           - |F|(s_i\kijk - 1)),
           \end{align*}
the Newton update for the parameter $\tau_i$ becomes
\begin{equation}
  \hat{\tau}_{i} = \hat{\tau}^{\text{old}}_{i}
    \exp\left(-\frac{
      g_{\tau_i}(\hat{\tau}_i^{\text{old}}, \bkappa_i)
    }{
      h_{\tau_i}(\hat{\tau}_i^{\text{old}}, \bkappa_i)\hat{\tau}_i^{\text{old}} +
     g_{\tau_i}(\hat{\tau}_i^{\text{old}}, \bkappa_i)
  }\right),
\label{eq:tau}
\end{equation}
where $\hat{\tau}_i^{\text{old}}$ is the estimate for
$\tau_i$ in the previous iteration of the Newton algorithm.

We obtain the Newton update for the parameter vector $\bkappa_i$  by
treating $\tau_i$ and $\bnu_i$ as fixed: their elements are  the quantities $\nu_{ijk}^f$ with $f = 1, \ldots,
|F|$, $j = 1, \ldots, |\PXi|$ and $k = 1, \ldots, |X_i|$. If we define
\begin{align*}
 & g_{\kijk}(\bkappa_i,\tau,\bnu_i)
    = {\partial \mathcal{L}_i}/{\partial \kijk} = \sum_{f = 1}^{|F|}s_i(\psi(\nu_{ijk}^f)-\psi(\nu_{i\cdot\cdot}^f)+ \tau_i\psi'(\tau_i\kijk)
         ([\balpha_0]_{ijk} +\\
         &- \tau_i \kijk -
          |F|(s_i\kijk-1)) +\tau_i\psi(\tau_i)+
     s_i|F|(\psi(\tau_i)-\psi(\tau_i\kijk) - \psi(s_i\kijk) +\\
     &+
       \log(\kijk) + 1) - \frac{|F|}{\kijk}
\end{align*}
and
\begin{align*}
  &h_{\kijk}(\bkappa_i,\tau_i,\bnu_i)
    = {\partial^2 \mathcal{L}_i}/{\partial \kappa^2_{ijk}} =  \tau_i^2\psi''(\tau_i\kijk)
       ([\balpha_0]_{ijk} - \tau_i \kijk +\\
       &-
        |F|(s_i\kijk - 1)) - \tau_i\psi'(\tau_i\kijk)(\tau_i+2s_i|F|) -
       s_i^2|F|\psi'(s_i\kijk) + \frac{s_i|F|}{\kijk} +
       \frac{|F|}{\kappa^2_{ijk}},
\end{align*}
we can write the Newton update for $\kijk$  as
\begin{equation}
  \hat{\kappa}_{ijk}
    = \hat{\kappa}_{ijk}^{\text{old}} +
      \frac{
        \displaystyle{
          \sum_{l=1}^{|X_i|} \sum_{\jmath=1}^{|\PXi|}
          \frac{
            g_{\kappa_{i\jmath l}}
            (\bkappa_i^{\text{old}},\tau_i,\bnu_i)
          }{
            h_{\kappa_{i\jmath l}}
            (\bkappa_i^{\text{old}},\tau_i,\bnu_i)
          }}
      }{
        \displaystyle{
          \sum_{l=1}^{|X_i|} \sum_{\jmath=1}^{|\PXi|}
          \frac{
            h_{\kijk}(\bkappa_i^{\text{old}},\tau_i,\bnu_i)
          }{
            h_{\kappa_{i\jmath l}}(\bkappa_i^{\text{old}},\tau_i,\bnu_i)
          }
        }
      } -
    \frac{
      g_{\kijk}
      (\bkappa_{i}^{\text{old}},\tau_i,\bnu_i)
    }{
      h_{\kijk}(\bkappa_{i}^{\text{old}},\tau_i,\bnu_i)
    },
  \label{eq:kappa}
\end{equation}
where $\kijk^{\text{old}}$ is the estimate for $\kijk$ in the
previous iteration of the Newton algorithm.

\section{Parameters of numerical example}
\label{app:param}
Tables \ref{tab:paramX1}-\ref{tab:paramX5} show the original and the estimated parameters of the network used in the numerical example presented in Section \ref{sec:exe}.

\begin{table}[H]
  \begin{center}
    \caption{Parameters associated with $X_1$.}
    \label{tab:paramX1}
    \begin{tabular}{|c|c|c||c|c|}
    \hline
    &  \multicolumn{2}{c||}{original} & \multicolumn{2}{c|}{estimated}\\
    \hline
    $X_1 | F$ & $1$ & $2$ & $1$ & $2$ \\
     \hline
    $1$ & 0.42 & 0.18 & 0.41 & 0.17 \\
    $2$ & 0.58 & 0.82 & 0.59 & 0.83 \\
      \hline
    \end{tabular}
    \caption{Parameters  associated with $X_2$.}
    \label{tab:paramX2}
    \begin{tabular}{|c|c|c|c|c||c|c|c|c|}
    \hline
    &  \multicolumn{4}{c||}{original} & \multicolumn{4}{c|}{estimated}\\
    \hline
     $F$ & \multicolumn{2}{c|}{$1$} & \multicolumn{2}{c||}{$2$} & \multicolumn{2}{c|}{$1$} & \multicolumn{2}{c|}{$2$} \\
     \hline
     $X_2 | X_1$ & $1$ & $2$& $1$ & $2$& $1$ & $2$& $1$ & $2$\\
     \hline
    $1$ & 0.31& 0.65 &  0.56 & 0.16 & 0.33 & 0.64 &  0.58 & 0.14\\
    $2$ & 0.69  &0.35 & 0.44 & 0.84 &0.67 & 0.36 & 0.42 & 0.86\\
      \hline
    \end{tabular}
    \caption{Parameters associated with  $X_3$.}
    \label{tab:paramX3}
    \begin{tabular}{|c|c|c|c|c||c|c|c|c|}
    \hline
    &  \multicolumn{4}{c||}{original} & \multicolumn{4}{c|}{estimated}\\
    \hline
     $F$ & \multicolumn{2}{c|}{$1$} & \multicolumn{2}{c||}{$2$} & \multicolumn{2}{c|}{$1$} & \multicolumn{2}{c|}{$2$} \\
     \hline
     $X_3 | X_1$ & $1$ & $2$& $1$ & $2$& $1$ & $2$& $1$ & $2$\\
     \hline
    $1$ & 0.61 & 0.42 &  0.15 & 0.53 &0.64 &0.39& 0.15 &0.54\\
    $2$ & 0.39 & 0.58 & 0.85 & 0.47 & 0.36 & 0.61& 0.85 &0.46\\
      \hline
    \end{tabular}
    \caption{Parameters associated with $X_4$.}
    \label{tab:paramX4}
    \begin{tabular}{|c|c|c|c|c||c|c|c|c|}
    \hline
    &  \multicolumn{4}{c||}{original} & \multicolumn{4}{c|}{estimated}\\
    \hline
     $F$ & \multicolumn{2}{c|}{$1$} & \multicolumn{2}{c||}{$2$} & \multicolumn{2}{c|}{$1$} & \multicolumn{2}{c|}{$2$} \\
     \hline
     $X_4 | X_3$ & $1$ & $2$& $1$ & $2$& $1$ & $2$& $1$ & $2$\\
     \hline
    $1$ & 0.47  & 0.51 &  0.38 & 0.28 &0.45 & 0.51& 0.40 & 0.28\\
    $2$ & 0.53 & 0.49 & 0.62 & 0.72 &0.55 & 0.49& 0.60 &0.72\\
      \hline
    \end{tabular}
 \end{center}
\end{table}
\begin{table}[H]
  \begin{center}
    \caption{Parameters associated with $X_5$.}
    \label{tab:paramX5}
    \begin{tabular}{|c|c|c|c|c|c|c|c|c|}
    \hline
    &  \multicolumn{8}{c|}{original}\\
    \hline
     $F$ & \multicolumn{4}{c|}{$1$} & \multicolumn{4}{c|}{$2$} \\
     \hline
     $X_4$ & \multicolumn{2}{c|}{$1$} & \multicolumn{2}{c|}{$2$} & \multicolumn{2}{c|}{$1$} & \multicolumn{2}{c|}{$2$}  \\
     \hline
     $X_5 | X_1$ & $1$ & $2$& $1$ & $2$ & $1$ & $2$& $1$ & $2$\\
     \hline
    $1$ & 0.46 & 0.56 &  0.57 & 0.52 &0.65& 0.43&0.22 &0.38 \\
    $2$ & 0.54 & 0.44 &0.43 & 0.48 & 0.35 &0.57&0.78 &0.62 \\
      \hline
      \hline
       & \multicolumn{8}{c|}{estimated}\\
       \hline
    $F$ & \multicolumn{4}{c|}{$1$} & \multicolumn{4}{c|}{$2$} \\
      \hline
     $X_4$ & \multicolumn{2}{c|}{$1$} & \multicolumn{2}{c|}{$2$} & \multicolumn{2}{c|}{$1$} & \multicolumn{2}{c|}{$2$}  \\
     \hline
     $X_5 | X_1$ & $1$ & $2$& $1$ & $2$ & $1$ & $2$& $1$ & $2$\\
     \hline
    $1$ & 0.41 &0.58 &0.53 & 0.53& 0.77 &0.44& 0.14 &0.38\\
    $2$ & 0.59 &0.42 &0.47 & 0.47& 0.23 &0.56& 0.86 &0.62\\
      \hline
    \end{tabular}
  \end{center}
\end{table}

\subsubsection*{Acknowledgements}
We would like to acknowledge support for this project from the Swiss National
Science Foundation (NSF, Grant No. IZKSZ2\_162188).

\bibliography{biblio}

\end{document}